\documentclass[twoside,11pt]{article}

\usepackage{color}
\usepackage{jmlr2e}
\usepackage{graphics}
\usepackage{epsfig}
\usepackage[ruled]{algorithm2e}
\usepackage{booktabs}
\usepackage{url}
\usepackage{psfrag}
\usepackage{relsize}
\usepackage{amsmath,amsfonts,mathrsfs,mathtools,amssymb}
\usepackage[colorlinks,linkcolor=blue,anchorcolor=green,citecolor=green]{hyperref}
\usepackage{tikz,pgfplots}
\usepackage{tkz-tab}
\usepackage[format=hang,justification=justified,singlelinecheck=false,font=footnotesize]{caption}
\usepackage[font=footnotesize]{subcaption}

\newcommand{\eins}{\boldsymbol{1}}
\newcommand{\vertiii}[1]{{\vert\kern-0.25ex\vert\kern-0.25ex\vert #1 
    \vert\kern-0.25ex\vert\kern-0.25ex\vert}}

\jmlrheading{}{}{~}{~}{~}{2016 Hanyuan Hang, Ingo Steinwart, Yunlong Feng and Johan A.K. Suykens}

\usepackage{tikz,pgfplots}
  \pgfplotsset{compat=newest}
  \pgfplotsset{plot coordinates/math parser=false,trim axis left}
     \newlength\figureheight
    \newlength\figurewidth

\begin{document}

\title{Kernel Density Estimation for Dynamical Systems }

\author{\name Hanyuan Hang \email hanyuan.hang@esat.kuleuven.be\\
\addr Department of Electrical Engineering, ESAT-STADIUS, KU Leuven\\
       Kasteelpark Arenberg 10, Leuven, B-3001, Belgium\\  
       \\     
       \name Ingo Steinwart \email ingo.steinwart@mathematik.uni-stuttgart.de\\
       \addr  Institute for Stochastics and Applications\\
              University of Stuttgart\\
              70569 Stuttgart, Germany\\
              \\
        \name Yunlong Feng \email yunlong.feng@esat.kuleuven.be\\	
        \name Johan A.K. Suykens \email johan.suykens@esat.kuleuven.be\\
\addr Department of Electrical Engineering, ESAT-STADIUS, KU Leuven\\
       Kasteelpark Arenberg 10, Leuven, B-3001, Belgium}

%\editor{ }

\maketitle

\begin{abstract}
We study the density estimation problem with observations generated by certain dynamical systems that admit a 
unique underlying invariant Lebesgue density. Observations drawn from dynamical systems are not independent and moreover, usual mixing concepts may not be appropriate for measuring the dependence among these observations. By employing the $\mathcal{C}$-mixing concept to measure the dependence, 
we conduct statistical analysis on the consistency and convergence of the kernel density estimator. 
Our main results are as follows: First, we show that with properly chosen bandwidth, the kernel density estimator is universally consistent under $L_1$-norm; Second, we establish convergence rates for the estimator with respect to several classes of dynamical systems under $L_1$-norm. 
In the analysis, the density function $f$ is only assumed to be H\"{o}lder continuous which is  
a weak assumption in the literature of nonparametric density estimation and also more realistic 
in the  dynamical system context. Last but not least, we prove that the same 
convergence rates of the estimator under $L_\infty$-norm and $L_1$-norm can be achieved 
when the density function is H\"{o}lder continuous, compactly supported and bounded.  The bandwidth selection problem of the kernel density estimator for dynamical system is also discussed in our study via numerical simulations.        
\end{abstract}

\begin{keywords}
Kernel density estimation, dynamical system, dependent observations, $\mathcal{C}$-mixing process, universal consistency, convergence rates, covering number, learning theory  
\end{keywords}

\section{Introduction}
Dynamical systems are now ubiquitous and are vital in modeling complex systems, especially when they admit recurrence relations. Statistical inference for dynamical systems has drawn continuous attention across various fields, the topics of which include parameter estimation, invariant measure estimation, forecasting, noise detection, among others. For instance, in the statistics and machine learning community, the statistical inference for certain dynamical systems have been recently studied in \cite{suykens2012artificial,suykens2000recurrent,suykens2002least,zoeter2005change,anghel2007forecasting,steinwart2009consistency,deisenroth2012expectation,mcgoff2015consistency,hang2015bernstein}, just to name a few. We refer the reader to a recent survey in \cite{mcgoff2015statistical} for a general depiction of this topic. The purpose of this study is to investigate the density estimation problem for dynamical systems via a classical nonparametric approach, i.e., kernel density estimation.  

The commonly considered density estimation problem can be stated as follows. Let $x_1$, $x_2$, $\ldots, x_n$ be observations drawn independently  from 
an unknown distribution $P$ on $\mathbb{R}^d$ with the density $f$. Density estimation is concerned with the estimation of the underlying density $f$. Accurate estimation of the density is important for many machine learning tasks such as regression, classification, and clustering problems and  also plays an important role in many real-world applications. Nonparametric density estimators are popular since weaker assumptions are applied to the underlying probability distribution. Typical nonparametric density estimators include the histogram and kernel density estimator. In this study, we are interested in the latter one, namely, \textit{kernel density estimator}, which is also termed as \textit{Parzen-Rosenblatt estimator} \citep{parzen1962estimation,rosenblatt1956remarks} and takes the following form
\begin{align}\label{general_kde}
f_n(x)= \frac{1}{nh^d}\sum_{i=1}^n K\left(\frac{x-x_i}{h}\right). 
\end{align}
Here, $h:=h_n>0$ is a bandwidth parameter and $K$ is a smoothing kernel. In the literature, point-wise and uniform consistency and convergence rates of the estimator $f_n$ to the unknown truth density $f$ under various distance measurements, e.g., $L_1, L_2, L_\infty$, have been established by resorting to the regularity assumptions on the smoothing kernel $\mathcal{K}$, the density $f$, as well as the decay of the bandwidth sequence $\{h_n\}$. Besides the theoretical concerns on the consistency and convergence rates, another practical issue one usually needs to address is the choice of the bandwidth parameter $h_n$, which is also called the \textit{smoothing parameter}. It plays a crucial role in the bias-variance trade-off in kernel density estimation. In the literature, approaches to choosing the smoothing parameter include least-squares cross-validation \citep{bowman1984alternative,rudemo1982empirical}, biased cross-validation \citep{scott1987biased}, plug-in method \citep{park1990comparison,sheather1991reliable}, the double kernel method \citep{devroye1989double},  as well as the method based on a discrepancy principle \citep{eggermont2001maximum}. We refer the reader to \cite{jones1996brief} for a general overview and to \cite{wand1994kernel,cao1994comparative,jones1996progress,devroye1997universal} for more detailed reviews. 

Note that studies on the kernel density estimator \eqref{general_kde} mentioned above heavily rely on the assumption that the observations are drawn in an i.i.d fashion. In the literature of statistics and machine learning, it is commonly accepted that the i.i.d assumption on the given data can be very much restrictive in real-world applications. Having realized this, researchers turn to weaken this i.i.d assumption by assuming that the observations are weakly dependent under various notions of weakly dependence which include $\alpha$-mixing, $\beta$-mixing, and $\phi$-mixing \citep{bradley2005basic}. 
There has been a flurry of work to attack this problem with theoretical and practical concerns, see e.g., 
\cite{masry1983probability,masry1986recursive,robinson1983nonparametric,tran1989recursive,tran1989l1,hart1990data,yu1993density} and \cite{hall1995bandwidth}, under the above notions of dependence. As a matter of fact, the assumed correlation among the observations complicates the  kernel density estimation problem from a technical as well as practical view and also brings inherent barriers. This is because, more frequently, the analysis on the consistency and convergence rates of the kernel density estimator  \eqref{general_kde}  is proceeded by decomposing the error term into bias and variance terms, which correspond to data-free and data-dependent error terms, respectively. The data-free error term can be tackled by using techniques from the approximation theory while the data-dependent error term is usually dealt with by exploiting arguments from the empirical process theory such as concentration inequalities. As a result, due to the existence of dependence among observations and various notions of the dependence measurement, the techniques, and results concerning the data-dependent error term are in general not universally applicable. On the other hand, it has been also pointed out that the bandwidth selection in kernel density estimation under dependence also departures from the independent case, see e.g., \cite{hart1990data,hall1995bandwidth}.   

In fact, when the observations  $x_1, x_2, \ldots, x_n \in\mathbb{R}^d$  
are generated by certain ergodic measure-preserving dynamical systems, 
the problem of kernel density estimation can be even more involved. 
To explain, let us consider a discrete-time ergodic measure-preserving dynamical system 
described by the sequence $(T^n)_{n\geq 1}$ of iterates of an unknown map $T: \Omega\rightarrow \Omega$ 
with $\Omega\subset \mathbb{R}^d$ and a unique invariant measure $P$ 
which possesses a density $f$ with respect to the Lebesgue measure
(rigorous definitions will be given in the sequel).  That is, we have 
\begin{align}\label{model_dynamical_systems}
x_i = T^i(x_0),\quad i=1,2,\ldots,n,
\end{align} 
where $x_0$ is the initial state. It is noticed that in this case the usual mixing concepts are not general enough 
to characterize the dependence among observations generated by \eqref{model_dynamical_systems} 
\citep{maume2006exponential,hang2015bernstein}. 
On the other hand, existing theoretical studies on the consistency and convergence rates of 
the kernel density estimator for i.i.d.~observations 
frequently assume that the density function $f$ is sufficiently smooth, e.g., 
first-order or even second-order smoothness. 
However, more often than not, this requirement can be stringent in the dynamical system context. 
For instance, the Lasota-Yorke map \citep{lasota1973existence} admits a density $f$ which only belongs to the space $BV$, 
i.e., functions of bounded variation. 
This is also the case for the $\beta$-map in Example \ref{beta_map} 
(see Subsection \ref{sec::subsec::dynamicalsys}). 
Therefore, studies on kernel density estimation mentioned above with dependent observations, 
in general, may not be applicable.
For more detailed comparison we refer to Section \ref{CandD}.

In this study, the kernel density estimation problem with observations generated by dynamical systems \eqref{model_dynamical_systems} is approached by making use of a more general concept for measuring the dependence of observations, namely, the so-called $\mathcal{C}$-mixing process (refer to Section \ref{sec::preliminary} for the definition). Proposed in \cite{maume2006exponential} and recently investigated in \cite{hang2015bernstein} and \cite{hanglearning2015}, the $\mathcal{C}$-mixing concept is shown to be more general and powerful in measuring dependence among observations generated by dynamical systems and can accommodate a large class of dynamical systems. Recently, a Bernstein-type exponential inequality for $\mathcal{C}$-mixing processes was established in \cite{hang2015bernstein} and its applications to some learning schemes were explored in \cite{hang2015bernstein} and \cite{hanglearning2015}.

Our main purpose in this paper is to conduct some theoretical analysis 
and practical implementations on the kernel density estimator for dynamical systems. 
The primary concern is the consistency 
and convergence rates of the kernel density estimator \eqref{general_kde} 
with observations generated by dynamical systems \eqref{model_dynamical_systems}. 
The consistency and convergence analysis is conducted under $L_1$-norm, and $L_\infty$-norm, respectively.  
% The $L_1$-norm is considered to be more than a reasonable choice in \cite{devroye1985nonparametric} for density estimation. 
% See also our discussion in Section \ref{sec::consistency_convergence} 
% for pros and cons of using $L_1$-norm in density estimation problems.
% Concerning the consistency of the kernel density estimator \eqref{general_kde}  
% with observations generated by dynamical systems \eqref{model_dynamical_systems}, 
We show that under mild assumptions on the smoothing kernel, 
with properly chosen bandwidth, the estimator is universally consistent under $L_1$-norm.
% i.e., its consistency does not rely on the probability distribution $P$ 
% (see Theorem \ref{ConsistencyL1} in Subsection \ref{sec::subsec::universal_consis});  
When the probability distribution $P$ 
possesses a polynomial or exponential decay outside of a radius-$r$ ball in its support, 
under the H\"{o}lder continuity assumptions on the kernel function and the density, 
we obtain almost optimal 
convergence rates under $L_1$-norm.
% can be established (see Theorem \ref{ConvergenceRatesL1} 
% in Subsection \ref{subsec::l1_convergence}); 
Moreover, when the probability distribution $P$ is compactly supported, 
which is a frequently encountered setting in the dynamical system context, 
we prove that stronger convergence results of the estimator can be developed, 
i.e., convergence results under $L_\infty$-norm 
% (see Theorem \ref{ConvergenceRatesLInfty} in Subsection \ref{subsec::l_infty_convergence}), 
which are shown to be of the same order with its $L_1$-norm convergence rates.
% (see Theorem  in Subsection \ref{subsec::l1_convergence});  
% As intermediate results in our analysis, we also show that the population version of the estimator  \eqref{general_kde} can approximate the density function $f$ well in the sense of Theorem \ref{ApproximationError} and \ref{ComparisionNorms}, under some mild assumptions on the kernel and the H\"{o}lder continuity of $f$. By taking into account also the tail behavior of the probability distribution $P$, the regularity of the transform $T$ and density function $f$, several oracle inequalities\footnote{Oracle inequalities are inequalities that \textit{compare the empirically obtained decision function with the one an omniscient oracle, having an infinite amount of observation, would obtain when pursuing the same goal.} See page $220$ in \cite{steinwart2008support}.} are then developed, in Theorem \ref{OracleInequalityL1}, Corollary \ref{oracle_inequality_corollary} and Theorem \ref{OracleInequalityInftyNorm}, which are crucial in deriving the convergence rates for the estimator;
Finally, with regard to the practical implementation of the estimator, we also discuss the bandwidth selection problem by performing numerical comparisons among several typical existing selectors that include least squares cross-validation and its variants for dependent observations as well as the double kernel method. We show that the double kernel bandwidth selector proposed in \cite{devroye1989double} can in general work well. Moreover, according to our numerical experiments, we find that bandwidth selection for kernel density estimator of dynamical systems 
is usually ad-hoc in the sense that its performance may depend on the considered dynamical system.

The rest of this paper is organized as follows. Section \ref{sec::preliminary} is a warm-up section 
for the introduction of some notations, definitions and assumptions that are related 
to the kernel density estimation problem and dynamical systems. Section \ref{sec::consistency_convergence} is concerned with the consistency and convergence of the kernel 
density estimator and presents the main theoretical results of this study. 
We discuss the bandwidth selection problem in Section \ref{sec::bandwidth_selection}.
All the proofs of Section \ref{sec::consistency_convergence} 
can be found in Section \ref{proofs}.
We end this paper in Section \ref{sec::conclusion}.

\section{Preliminaries}\label{sec::preliminary}
 
\subsection{Notations}
Throughout this paper, $\lambda^d$ is denoted as the Lebesgue measure on $\mathbb{R}^d$ and   $\|\cdot\|$ is an arbitrary norm on $\mathbb{R}^d$. We denote $B_r$ as the centered ball of $\mathbb{R}^d$ with radius $r$, that is, 
\begin{align*}
 B_r := \{ x = (x_1, \ldots, x_d) \in \mathbb{R}^d : \|x\| \leq r \},
\end{align*}
and its complement $H_r$ as
\begin{align*}
 H_r := \mathbb{R}^d \big\backslash B_r
     = \{ x \in \mathbb{R}^d : \|x\| > r \}.
\end{align*}
Recall that for $1 \leq p < \infty$, the $\ell_p^d$-norm is defined as
$\|x\|_{\ell_p^d} := ( x_1^p + \cdots + x_d^p )^{1/p}$,
and the $\ell_{\infty}^d$-norm is defined as
$\|x\|_{\ell_{\infty}^d} := \max_{i = 1, \ldots, d} |x_i|$.  Let $(\Omega, \mathcal{A}, \mu)$ be a probability space.
We denote $L_p(\mu)$ as  the space of (equivalence classes of) measurable functions $g : \Omega \to \mathbb{R}$ with finite $L_p$-norm $\|g	\|_p$. Then $L_p(\mu)$ together with $\|g\|_p$ forms a Banach space. Moreover, if  $\mathcal{A}'\subset \mathcal{A}$ is a sub-$\sigma$-algebra, then $L_p(\mathcal A', \mu)$ denotes the space of all $\mathcal{A}'$-measurable functions $g\in L_p(\mu)$. Finally, for a Banach space $E$, we write $B_E$ for its closed unit ball.

In what follows, the notation $a_n \lesssim b_n$ means that there exists a positive constant $c$ 
such that $a_n \leq c \, b_n$, for all $n \in \mathbb{N}$. 
% We also denote $a_n \sim b_n$ whenever $a_n = \mathcal{O}(b_n)$. 
With a slight abuse of notation, in this paper, $c, c^\prime$ and $C$ are used interchangeably for positive constants while their values may vary across different lemmas, theorems, and corollaries.

\subsection{Dynamical Systems and $\mathcal{C}$-mixing Processes}\label{sec::subsec::dynamicalsys}
In this subsection, we first introduce the dynamical systems of interest, namely, ergodic measure-preserving dynamical systems. Mathematically, an ergodic measure-preserving \textit{dynamical system} is a system $(\Omega, \mathcal{A}, \mu, T)$ with a mapping $T: \Omega \rightarrow \Omega$ that is measure-preserving, i.e., $\mu(A)=\mu(T^{-1}A)$ for all $A\in\mathcal{A}$, and ergodic, i.e., $T^{-1} A = A$ implies $\mu(A)=0$ or $1$. In this study, we are confined to the dynamical systems in which $\Omega$ is a subset of $\mathbb{R}^d$, $\mu$ is a probability measure that is absolutely continuous with respect to the Lebesgue measure $\lambda$ and admits a unique invariant Lebesgue density $f$. 

In our study, it is assumed that the observations $x_1,x_2,\cdots,x_n$ are generated by the discrete-time dynamical system \eqref{model_dynamical_systems}. Below we list several typical examples of discrete-time dynamical systems that satisfy the above assumptions  \citep{lasota1985probabilistic}:

\begin{example}[Logistic Map]\label{logistic_map}
The Logistic map is defined by	
$$T(x)=\lambda x(1-x), \,\,\,x\in (0,1), \,\,\,\lambda\in [0,4],$$
with a unique invariant Lebesgue density
$$f(x)=\frac{1}{\pi\sqrt{x(1-x)}},\,\,\,0<x<1.$$
\end{example}

\begin{example}[Gauss Map]\label{gauss_map}
The Gauss map is defined by
$$T(x)=\frac{1}{x} \mod \,1,\,\,\,\,x\in(0,1),$$
with a unique invariant Lebesgue density
$$f(x)=\frac{1}{\log 2}\cdot\frac{1}{1+x},\,\,x\in(0,1).$$
\end{example}

\begin{example}[$\beta$-Map]\label{beta_map}
For $\beta>1$, the $\beta$-map is defined as 
$$T(x)=\beta x \mod \, 1,\,\,x\in (0,1),$$ 
with a unique invariant Lebesgue density given by 
$$f(x)=c_\beta\sum_{i\geq 0}\beta^{-(i+1)}\boldsymbol{1}_{[0, T^i (1)]}(x),$$
where $c_\beta$ is a constant chosen such that $f$ has integral $1$. 
\end{example}
 
We now introduce the notion for measuring the dependence among observations from dynamical systems, 
namely, $\mathcal{C}$-mixing process  \citep{maume2006exponential,hang2015bernstein}. 
To this end, let us assume that $(X, \mathcal{B})$ is a measurable space with $X\subset\mathbb{R}^d$. 
Let $\mathcal{X} := (X_n)_{n \geq 1}$ be an $X$-valued stochastic process on $(\Omega, \mathcal A, \mu)$, 
and for $1 \leq i \leq j \leq \infty$, denote by $\mathcal{A}_i^j$ the $\sigma$-algebra 
generated by $(X_i, \ldots, X_j)$. Let $\varGamma:\Omega\rightarrow X$ be a measurable map. 
$\mu_\mathsmaller{\varGamma}$ is denoted as the $\varGamma$-image measure of $\mu$, 
which is defined as $\mu_\mathsmaller{\varGamma}(B):= \mu(\varGamma^{-1}(B))$, $B\subset X$ measurable. 
The process  $\mathcal{X}$ is called \textit{stationary} if 
$\mu_{(X_{i_1+j}, \ldots, X_{i_n+j})} = \mu_{(X_{i_1}, \ldots, X_{i_n})}$ for all $n, j, i_1, \ldots, i_n \geq 1$. 
Denote $P := \mu_{X_1}$. Moreover, for $\psi, \varphi \in L_1(\mu)$ satisfying $\psi\varphi\in L_1(\mu)$, 
we denote the correlation of $\psi$ and $\varphi$ by 
\begin{align*}
\mathrm{cor} (\psi, \varphi) 
:= \int_{\Omega} \psi \, \varphi \, \mathrm{d} \mu -
   \int_{\Omega} \psi \, \mathrm{d} \mu \cdot \int_{\Omega} \varphi \, \mathrm{d} \mu\, .
\end{align*} 
It is shown that several dependency coefficients for $\mathcal{X}$ 
can be expressed in terms of such correlations for restricted sets of functions $\psi$ and $\varphi$ \citep{hang2015bernstein}. In order to introduce the  notion, we also need to define a new norm, which is taken from \cite{maume2006exponential} 
and introduces  restrictions on $\psi$ and $\varphi$ considered here. 
Let us assume that $\mathcal{C}(X)$ is a subspace 
of bounded measurable functions $g : X \rightarrow \mathbb{R}$ and that
we have a semi-norm $\vertiii{\cdot}$ on $\mathcal{C}(X)$. 
For $g \in \mathcal{C}(X)$, we define the $\mathcal{C}$-norm $\|\cdot\|_\mathcal{C}$ by
\begin{align} \label{lambdanorm}
\|g\|_{\mathcal{C}} := \|g\|_{\infty} + \vertiii{g}.
\end{align} 
Additionally, we need to introduce the following restrictions on the semi-norm $\vertiii{\cdot}$.
\begin{assumption}\label{seminorm_assump}
We assume that the following two restrictions on the semi-norm $\vertiii{\cdot}$ hold:
\begin{itemize}
\item [(i)]   $ \vertiii{ g }=0$
         for all constant functions $g \in \mathcal{C}(X)$;
\item [(ii)] $ \vertiii{ e^g }  \leq  \bigl\| e^g \bigr\|_{\infty} \vertiii{g},  \,\,\,\,\,\, \quad  g \in \mathcal{C}(X)$. 
\end{itemize}
\end{assumption}
Note that the first constraint on the semi-norm in Assumption \ref{seminorm_assump} implies its shift invariance on $\mathbb{R}$ while the inequality constraint can be viewed as an abstract  \textit{chain rule}  if one views the semi-norm as a norm describing aspects of the smoothness of $g$, as discussed in \cite{hang2015bernstein}.  In fact, it is easy to show that  the following function classes, which are probably also the most frequently considered in the dynamical system context, satisfy Condition $(i)$ in Assumption \ref{seminorm_assump}. Moreover, they also satisfy Condition $(ii)$ in Assumption \ref{seminorm_assump}, as shown in \citep{hang2015bernstein}:	
\begin{itemize}
\item $L_\infty(X)$: The class of bounded functions on $X$;
\item  $BV(X)$:    The  class of bounded variation functions on $X$;  
\item   $C_{b,\alpha}(X)$:    The   class of bounded and $\alpha$-H\"{o}lder continuous functions on  $X$; 
\item   $\text{Lip}(X)$:     The   class of Lipschitz continuous functions  on  $X$; 
\item  $C^1(X)$:    The class of continuously differentiable functions  on $X$.
\end{itemize}

\begin{definition}[\bf $\mathcal{C}$-mixing Process] 
Let $(\Omega, \mathcal A, \mu)$ be a probability space, $(X, \mathcal{B})$ be a measurable space,
$\mathcal{X}:=(X_i)_{i \geq 1}$ be an $X$-valued, stationary process on $\Omega$, and 
$\|\cdot\|_{\mathcal{C}}$ be defined by \eqref{lambdanorm} for some semi-norm $\vertiii{\cdot}$.
Then, for $n \geq 1$, we define the $\mathcal{C}$-mixing coefficients by
\begin{align*}  
\phi_{\mathcal{C}}(\mathcal{X}, n)
:= \sup \big \{ \mathrm{cor}(\psi, g(X_{k+n})) : k \geq 1, \psi \in B_{L_1(\mathcal A_1^k, \mu)}, g \in B_{\mathcal{C}(X)}
\big\},
\end{align*} 
and the time-reversed $\mathcal{C}$-mixing coefficients by
\begin{align*} 
\phi_{\mathcal{C}, \text{rev}}(\mathcal{X}, n)
:= \sup \big\{ 
\mathrm{cor}(g(X_k), \varphi) : 
k \geq 1, g \in B_{\mathcal{C}(X)}, \varphi \in B_{L_1(\mathcal A_{k+n}^{\infty}, \mu)}  \big\}.
\end{align*}
Let $(d_n)_{n \geq 1}$ be a strictly positive sequence converging to $0$. 
We say that \textbf{$\mathcal{X}$  is \emph{(time-reversed) $\mathcal{C}$-mixing}} with rate $(d_n)_{n \geq 1}$, 
if we have $\phi_{\mathcal{C},\text{rev}}(\mathcal{X}, n) \leq d_n$ for all $n \geq 1$.
Moreover, if $(d_n)_{n \geq 1}$ is of the form
\begin{align*}   
d_n := c_0 \exp \bigl( - b n^{\gamma} \bigr), ~~~~~~ n \geq 1, 
\end{align*}
for some constants $c_0 > 0$, $b > 0$, and $\gamma > 0$, then $\mathcal{X} $ 
is called  \textbf{\emph{geometrically} (time-reversed) $\mathcal{C}$-mixing}. 
\end{definition} 

From the above definition, we see that a $\mathcal{C}$-mixing process is defined in association with an underlying function space. For the above listed function spaces, i.e., $L_\infty(X)$, $BV(X)$, $C_{b,\alpha}(X)$, $\text{Lip}(X)$ and $C^1(X)$, the increase of the smoothness enlarges the class of the associated stochastic processes, as illustrated in \cite{hang2015bernstein}. Note that the classical $\phi$-mixing process is essentially a $\mathcal{C}$-mixing process associated with the function space $L_\infty(X)$. Note also that not all $\alpha$-mixing processes are $\mathcal{C}$-mixing, and vice versa.  We refer the reader to \cite{hang2015bernstein} for the relations among $\alpha$-, $\phi$- and $\mathcal{C}$-mixing processes.

On the other hand, under the above notations and definitions, from Theorem $4.7$ in \cite{maume2006exponential}, we know that Logistic map in Example \ref{logistic_map} is geometrically time-reversed $\mathcal{C}$-mixing with $\mathcal{C}=\text{Lip}(0,1)$ while Theorem $4.4$ in  \cite{maume2006exponential} (see also Chapter $3$ in \cite{baladi2000positive}) indicates that Gauss map in Example \ref{gauss_map} is geometrically time-reversed $\mathcal{C}$-mixing with $\mathcal{C}=BV(0,1)$. Example \ref{beta_map} is also geometrically time-reversed $\mathcal{C}$-mixing with $\mathcal{C}=BV(0,1)$ according to \cite{maume2006exponential}. For more examples of geometrically time-reversed $\mathcal{C}$-mixing dynamical systems, the reader is referred to Section $2$ in \cite{hang2015bernstein}.

\subsection{Kernel Density Estimation: Assumptions and Formulations} 
For the smoothing kernel $K$ in the kernel density estimator, in this paper we consider its following general form, namely, $d$-dimensional smoothing kernel: 
 
\begin{definition}\label{kernel}
A bounded, monotonically decreasing function $K : [0, \infty) \to [0, \infty)$ is  a \textbf{$d$-dimensional smoothing kernel} if 
\begin{align}  \label{kernelFormular}
\int_{\mathbb{R}^d} K(\|x\|) \, \mathrm{d}x =: \kappa \in (0, \infty).
\end{align}
\end{definition}

The choice of the norm in Definition \ref{kernel} does not matter since all norms on $\mathbb{R}^d$ are equivalent. To see this, let $\|\cdot\|'$ be another norm on $\mathbb{R}^d$ satisfying $\kappa \in (0, \infty)$. From the equivalence of the two norms on $\mathbb{R}^d$, one can find a positive constant  $c$ such that $\|x\| \leq c \|x\|'$ holds for all $x \in \mathbb{R}$. Therefore, easily we have
\begin{align*}
\int_{\mathbb{R}^d} K(\|x\|') \, \mathrm{d}x
\leq \int_{\mathbb{R}^d} K \left( \|x\|/c  \right) \, \mathrm{d}x
 = c^d \int_{\mathbb{R}^d} K(\|x\|) \, \mathrm{d}x < \infty.
\end{align*}
In what follows, without loss of generality, we assume that the constant $\kappa$ in Definition \ref{kernel} equals to $1$.

\begin{lemma}\label{lemma_kernel}
A bounded, monotonically decreasing function $K : [0, \infty) \to [0, \infty)$
is a $d$-dimensional smoothing kernel if and only if
\begin{align*}
\int_0^{\infty} K(r) r^{d-1} \, \mathrm{d}r \in (0, \infty).
\end{align*}
\end{lemma}
\begin{proof}
From the above discussions, it suffices to consider
the integration constraint for the kernel function $K$ with respect to  the Euclidean norm
$\|\cdot\|_{\ell_2^d}$. We thus have
\begin{align*}
\int_{\mathbb{R}^d} K \big( \|x\|_{\ell_2^d} \big) \, \mathrm{d}x
= d \tau_d \int_0^{\infty} K(r) r^{d-1} \, \mathrm{d}r,
\end{align*}
where $\tau_d = \pi^{d/2}\big/\Gamma \big( \frac{d}{2} + 1 \big)$
is the volume of the unit ball $B_{\ell_2^d}$ of the Euclidean space $\ell_2^d$. This completes the proof of Lemma \ref{lemma_kernel}.
\end{proof}

Let $r\in [0,+\infty)$ and denote $\boldsymbol{1}_A$ as the indicator function. Several common examples of $d$-dimensional smoothing kernels $K(r)$ include the Naive kernel $\eins_{[0,1]}(r)$, the Triangle kernel $(1 - r) \eins_{[0,1]}(r)$, the Epanechnikov kernel $(1 - r^2) \eins_{[0,1]}(r)$, and the Gaussian kernel $e^{-r^2}$. In this paper, we are interested in the kernels that satisfy the following restrictions on their shape and regularity:
\begin{assumption}\label{assump_kernel}
For a fixed function space $\mathcal{C}(X)$, we make the following assumptions on the $d$-dimensional smoothing kernel $K$:
\begin{itemize}
\item [(i)] $K$ is H\"{o}lder continuous with exponent $\beta$ with $\beta\in[0,1]$;
\item [(ii)]
$ \int_0^{\infty} K(r) r^{\beta+d-1} \, \mathrm{d}r < \infty$;
\item [(iii)] For all $x \in \mathbb{R}^d$, we have $\vertiii{K(\|x - \cdot\|/h)} \in \mathcal{C}(X)$ and
there exists a function $\varphi : (0, \infty) \to (0, \infty)$ such that 
$$\sup_{x \in \mathbb{R}^d}\vertiii{K(\|x - \cdot\|/h)} \leq  \varphi(h).$$ 
\end{itemize}
\end{assumption}

It is easy to verify that for  $\mathcal{C} = \text{Lip}$,
Assumption \ref{assump_kernel} is met for 
the Triangle kernel,
the Epanechnikov kernel, and the Gaussian kernel.
Particularly, Condition $(iii)$ holds for all these kernels 
with $\vertiii{\cdot}$ being the Lipschitz norm and $\varphi(h)\leq \mathcal{O}(h^{-1})$. 
Moreover, as we shall see below, not all the conditions in Assumption  \ref{assump_kernel} 
are required for the analysis conducted in this study and conditions assumed on the kernel will be specified explicitly. 

We now show that given a $d$-dimensional smoothing kernel $K$ as in Definition \ref{kernel}, one can easily construct a probability density on $\mathbb{R}^d$.

\begin{definition}[$K$-Smoothing of a Measure]   
Let $K$ be a $d$-dimensional smoothing kernel and $Q$ be a probability measure on $\mathbb{R}^d$.
Then, for $h > 0$,
\begin{align*}
f_{Q,h}(x):=f_{Q,K,h}(x) :  
 = h^{-d} \int_{\mathbb{R}^d} K \left(  \|x-x'\|/h \right) \,\mathrm{d}Q(x'),
 \,\,\,\,\,\,\,\,
 x\in\mathbb{R}^d,
\end{align*}
is called a \textbf{$K$-smoothing of $Q$}.
\end{definition}

It is not difficult to see that $f_{Q,h}$ defines a probability density on $\mathbb{R}^d$, since
Fubini's theorem yields that
\begin{align*}
\int_{\mathbb{R}^d} f_{Q,h}(x) \, \mathrm{d}x
& = \int_{\mathbb{R}^d} \int_{\mathbb{R}^d} h^{-d} K \left(  \|x-x'\|/h \right) \, \mathrm{d}Q(x') \, \mathrm{d}x
\\
& = \int_{\mathbb{R}^d} \int_{\mathbb{R}^d} K(\|x\|) \, \mathrm{d}x \, \mathrm{d}Q(x') = 1.
\end{align*}
Let us denote $K_h: \mathbb{R}^d\rightarrow [0,+\infty)$ as 
\begin{align} \label{K_h}
K_h(x) : = h^{-d}K\left(  \|x\|/h\right),\,\,x\in\mathbb{R}^d.
\end{align} 
Note that $K_h$ also induces a density function on $\mathbb{R}^d$ since there holds $\|K_h\|_1 = 1$.

For the sake of notational simplification, in what follows, we introduce the convolution operator $*$. 
% Some properties of the convolution operator can be found in Lemma \ref{appendix_convolution} 
% in the Appendix and also Chapter 8 in \cite{Folland99}. 
Under this notation, we then see that $f_{Q,h}$ is the density of the measure 
that is the convolution of the  measure $Q$ and $\nu_h = K_h \, \mathrm{d}\lambda^d$. 
Recalling that $P$ is a probability measure on $\mathbb{R}^d$ 
with the corresponding density function $f$, by taking $Q:=P$ with $\mathrm{d} P = f \, \mathrm{d}\lambda^d$, we have 
\begin{align}\label{KDE_population_version}
f_{P,h} = K_h * f = f* K_h=K_h*\mathrm{d}P.
\end{align}
Since $K_h \in L_\infty(\mathbb{R}^d)$ and $f \in L_1(\mathbb{R}^d)$, 
from Proposition (8.8) in \cite{Folland99}
we know that $f_{P,h}$ is uniformly continuous and bounded. Specifically, when $Q$ is the empirical measure $D_n  = \frac{1}{n} \sum_{i=1}^n \delta_{x_i}$, the \textit{kernel density estimator for dynamical systems} in this study can be expressed as
\begin{align}\label{KDE_formal}
\begin{split}
f_{D_n,h}(x)   = K_h* \mathrm{d} D_n(x) 
             = \frac{1}{n h^d} \sum_{i=1}^n K \bigg( \frac{\|x-x_i\|}{h} \bigg).
\end{split}
\end{align}
From now on, for notational simplicity, we will suppress the subscript $n$ of $D_n$ and denote $D:=D_n$, e.g., $f_{D,h} := f_{D_n,h}$.

\section{Consistency and Convergence Analysis}\label{sec::consistency_convergence}
In this section, we study the consistency and convergence rates of $f_{D,h}$ to the true density $f$ under $L_1$-norm and also $L_\infty$-norm for some special cases. Recall that $f_{D,h}$ is a nonparametric density estimator and so the criterion that measures its goodness-of-fit matters, which, for instance, includes $L_1$-distance, $L_2$-distance and $L_\infty$-distance.

In the literature of kernel density estimation, probably the most frequently employed criterion is the $L_2$-distance of the difference between $f_{D,h}$ and $f$, since it entails an exact bias-variance decomposition and can be analyzed relatively easily by using Taylor expansion involved arguments. However, it is 	argued in \cite{devroye1985nonparametric} (see also \cite{devroye2001combinatorial}) that  $L_1$-distance could be a more reasonable choice since: it is invariant under monotone transformations; it is always well-defined as a metric on the space of density functions; it is also proportional to the total variation metric and so leads to better visualization of the closeness to the true density function than $L_2$-distance. The downside of using $L_1$-distance is that it does not admit an exact bias-variance decomposition and the usual Taylor expansion involved techniques for error estimation may not apply directly. Nonetheless, if we introduce the intermediate estimator $f_{P,h}$ in \eqref{KDE_population_version}, obviously the following inequality holds 
\begin{align}\label{error decomposition}
\|f_{D,h} - f\|_1 \leq \|f_{D,h} - f_{P,h}\|_1 + \|f_{P,h} - f\|_1.
\end{align}
The consistency and convergence analysis in our study will be mainly conducted in the $L_1$ sense with the help of inequality \eqref{error decomposition}. Besides, for some specific case, i.e., when the density $f$ is compactly supported, we are also concerned with the consistency and convergence of $f_{D,h}$ to $f$ under $L_\infty$-norm. In this case, there also holds the following inequality
\begin{align}\label{error decomposition_infinity}
\|f_{D,h} - f\|_\infty \leq \|f_{D,h} - f_{P,h}\|_\infty + \|f_{P,h} - f\|_\infty.
\end{align}

It is easy to see that the first error term on the right-hand side  of \eqref{error decomposition}  
or \eqref{error decomposition_infinity}  is stochastic due to the empirical measure $D$ while the second one is deterministic because of its sampling-free nature. Loosely speaking, the first error term corresponds to the variance of the estimator $f_{D,h}$, while the second one can be treated as its bias although \eqref{error decomposition} or \eqref{error decomposition_infinity} is not an exact error decomposition. In our study, we proceed with the consistency and convergence analysis on $f_{D,h}$ by bounding the two error terms, respectively.

\subsection{Bounding the Deterministic Error Term}

Our first theoretical result on bounding the deterministic error term shows that, given a $d$-dimensional kernel $K$, the $L_1$-distance between its $K$-smooth of the measure $P$, i.e., $f_{P,h}$, and $f$ can be arbitrarily small by choosing the bandwidth appropriately.  Moreover, under mild assumptions on the regularity of $f$ and $K$, the $L_\infty$-distance between the two quantities possesses a polynomial decay with respect to the bandwidth $h$. 
	
\begin{theorem} \label{ApproximationError}
Let $K$ be a $d$-dimensional smoothing kernel. 
\begin{itemize}
 \item [(i)]
 For any $\varepsilon > 0$, there exists $0<h_\varepsilon \leq 1$
such that for any $h \in (0, h_\varepsilon]$ we have
\begin{align*}
\|f_{P,h} - f\|_1 \leq \varepsilon.
\end{align*}
 \item [(ii)]
 If $K$ satisfies Condition $(ii)$ in Assumption \ref{assump_kernel} and
$f$ is $\alpha$-H\"{o}lder continuous with $\alpha\leq \beta$,
then there holds
\begin{align*}
\|f_{P,h} - f\|_\infty \lesssim   h^{\alpha}.
\end{align*}
\end{itemize}

\end{theorem}

% Note that in Theorem \ref{ApproximationError}, when bounding the $L_1$-distance between $f_{P,h}$ and the density function $f$, no restrictions are imposed on the true density as well as the kernel function, except that $K$ is a $d$-dimensional smoothing kernel. When some further regularity restrictions are made on $f$, the following theorem can be established and shows that the $L_\infty$-distance between $f_{P,h}$ and $f$ possesses a polynomial decay with respect to the bandwidth $h$.   

We now show that the $L_1$-distance between $f_{P,h}$ and $f$ can be upper bounded by their difference (in the sense of $L_\infty$-distance) on a compact domain of $\mathbb{R}^d$ together with their difference (in the sense of $L_1$-distance) outside this domain. As we shall see later, this observation will entail us to consider different classes of the true densities $f$. The following result is crucial in our subsequent analysis on the consistency and convergence rates of $f_{D,h}$.  

\begin{theorem}\label{ComparisionNorms}
Assume that $K$ is a $d$-dimensional smoothing kernel that satisfies Conditions $(i)$ and $(ii)$ in Assumption \ref{assump_kernel}. For $h \leq 1$ and $r \geq 1$, we have
\begin{align*}
\|f_{P,h} - f\|_1
\lesssim   r^d \|f_{P,h} - f\|_{\infty} +   P ( H_{r/2} )
     +   ( h/r )^{\beta}.
\end{align*}
\end{theorem}

\subsection{Bounding the Stochastic Error Term}
We now proceed with the estimation of the stochastic error term $\|f_{D,h} - f_{P,h}\|_1$ 
by establishing probabilistic oracle inequalities. For the sake of readability, 
let us start with an overview of the analysis conducted in this subsection for bounding the stochastic error term. 

\subsubsection{An Overview of the Analysis}
In this study, the stochastic error term is tackled by using capacity-involved arguments 
and the Bernstein-type inequality established in \cite{hang2015bernstein}. 
In the sequel, for any fixed $x\in\Omega \subset \mathbb{R}^d$, we write 
\begin{align}\label{k_x_h}
k_{x,h}:= h^{-d} K  (  \|x - \cdot\|/h ),
\end{align}
and we further denote the centered random variable $\widetilde{k}_{x,h}$ on $\Omega$ as 
\begin{align}\label{k_x_h_mean}
\widetilde{k}_{x,h}  := k_{x,h} - \mathbb{E}_P k_{x,h}.
\end{align}
It thus follows that
\begin{align*}
\mathbb{E}_D \widetilde{k}_{x,h} 
= \mathbb{E}_D k_{x,h} - \mathbb{E}_P k_{x,h} 
= f_{D,h}(x) - f_{P,h}(x),
\end{align*}
and consequently we have
\begin{align*}
\|f_{D,h} - f_{P,h}\|_1
= \int_{\mathbb{R}^d} |\mathbb{E}_D \widetilde{k}_{x,h}| \, \mathrm{d}x,
\end{align*}
and
\begin{align*}
\|f_{D,h} - f_{P,h}\|_{\infty}
= \sup_{x\in\Omega} |\mathbb{E}_D \widetilde{k}_{x,h}|.
\end{align*}
As a result, in order to bound $\|f_{D,h} - f_{P,h}\|_1$, it suffices to bound the supremum of the empirical process  $ \mathbb{E}_D \widetilde{k}_{x,h} $ indexed by $x\in \mathbb{R}^d$. For any $r>0$, there holds
\begin{align*}
\|f_{D,h} - f_{P,h}\|_1
= \int_{B_r} |\mathbb{E}_D \widetilde{k}_{x,h}| \, \mathrm{d}x+\int_{H_r} |\mathbb{E}_D \widetilde{k}_{x,h}| \, \mathrm{d}x.	 
\end{align*}
The second term of the right-hand side of the above equality can be similarly dealt with 
as in the proof of Theorem \ref{ComparisionNorms}. 
In order to bound the first term, we define $\widetilde{\mathcal{K}}_{h,r}$ as the function set of $\widetilde{k}_{x,h}$ that corresponds to $x$ which lies on a radius-$r$ ball of $\mathbb{R}^d$:
\begin{align*} 
\widetilde{\mathcal{K}}_{h,r} := \bigl\{ \widetilde{k}_{x,h} : x \in B_r \bigr\} \subset L_\infty(\mathbb{R}^d).
\end{align*}
The idea here is to apply capacity-involved arguments and the Bernstein-type exponential inequality 
in \cite{hang2015bernstein}
to the function set $\widetilde{\mathcal{K}}_{h,r}$ and the associated empirical process $\mathbb{E}_{D}\widetilde{k}_{x,h}$. The difference between $f_{D,h}$ and $f_{P,h}$ under the $L_\infty$-norm can be bounded analogously. Therefore, to further our analysis, we first need to bound the capacity of $\widetilde{\mathcal{K}}_{h,r}$ in terms of covering numbers.  

\subsubsection{Bounding the Capacity of the Function Set $\widetilde{\mathcal{K}}_{h,r}$}

\begin{definition}[\bf Covering Number]
Let $(X, d)$ be a metric space and $A \subset X$.
For $\varepsilon > 0$, the \textbf{$\varepsilon$-covering number} of $A$ is denoted as 
\begin{align*}
\mathcal{N}(A,d,\varepsilon)
:= \min \left\{ n \geq 1 : \exists \,x_1, \cdots, x_n \in X   \, \,\text{such that}  \,\, \, A \subset \bigcup_{i=1}^n B_d(x_i,\varepsilon) \right\},
\end{align*}
where  $B_d(x,\varepsilon)
:= \left\{ x' \in X : d(x, x') \leq \varepsilon \right\}$.
\end{definition}

For a fixed $r \geq 1$, we consider the function set
\begin{align*}
\mathcal{K}_{h,r} := \{ k_{x,h} : x \in B_r \} \subset L_\infty(\mathbb{R}^d).
\end{align*}
The following proposition provides an estimate of the covering number of $\mathcal{K}_{h,r}$.

\begin{proposition}\label{kernelSpaceCN}
Let $K$ be a $d$-dimensional smoothing kernel that satisfies Conditions $(i)$ in Assumption \ref{assump_kernel} 
and $h \in (0,1]$.
 Then there exists a positive constant $c^\prime$ such that
 for all $\varepsilon \in (0,1]$, we have 
\begin{align*}
\mathcal{N} (\mathcal{K}_{h,r}, \|\cdot\|_\infty, \varepsilon)
\leq c^\prime r^d h^{-d-\frac{d^2}{\beta}} \varepsilon^{-\frac{d}{\beta}}.  
\end{align*}
\end{proposition}

\subsubsection{Oracle Inequalities under $L_1$-Norm, and $L_\infty$-Norm}\label{oracle_inequalites}

We now establish oracle inequalities for the kernel density estimator \eqref{KDE_formal} under $L_1$-norm, and $L_\infty$-norm, respectively. These oracle inequalities will be crucial in establishing the consistency and convergence results of the estimator. Recall that the considered kernel density estimation problem is based on samples from an $X$-valued $\mathcal{C}$-mixing process which is associated with an underlying function class $\mathcal{C}(X)$. As shown below, the established oracle inequality holds without further restrictions on the support of the density function $f$.

\begin{theorem}\label{OracleInequalityL1}
Suppose that Assumption \ref{assump_kernel} holds. Let $\mathcal X := (X_n)_{n \geq 1}$ 
be an $X$-valued stationary geometrically (time-reversed) $\mathcal{C}$-mixing process on $(\Omega, \mathcal A, \mu)$ with $\|\cdot\|_{\mathcal{C}}$ being defined for some semi-norm $\vertiii{\cdot}$ 
that satisfies Assumption \ref{seminorm_assump}. Then  for all $0 < h \leq 1$, $r \geq 1$ and $\tau \geq 1$, there exists an $n_0 \in \mathbb{N}$ such that for all $n \geq n_0$, 
with probability $\mu$  at least $1 - 3e^{-\tau}$, there holds
\begin{align*}
\|f_{D,h} - f_{P,h}\|_1
\lesssim       \sqrt{\frac{(\log n)^{2/\gamma} r^d (\tau + \log \frac{nr}{h})}{h^d n}}
        + \frac{(\log n)^{2/\gamma} r^d (\tau + \log \frac{nr}{h})}{h^d n}  
\\
\phantom{=}        
     +  P (H_{r/4})
     + \sqrt{\frac{32\tau(\log n)^{2/\gamma}}{n}} +  \biggl( \frac{h}{r} \biggr)^{\beta}.
	\end{align*}
Here $n_0$ will be given explicitly in the proof.
\end{theorem}

% When the density function $f$ is compactly supported, 
% restrictions on the underlying function space of the $\mathcal{C}$-mixing process can be relaxed. 
% That is, function spaces of smoother functions, such as $C_{b,\alpha}(X)$, $\text{Lip}(X)$, and $C^1(X)$, can be considered.  In this case, a more general oracle inequality can be established as follows.
% \begin{corollary}\label{oracle_inequality_corollary}  
% Suppose that Assumption \ref{assump_kernel} holds. Let $\mathcal X := (X_n)_{n \geq 1}$ be an $X$-valued stationary geometrically (time-reversed) $\mathcal{C}$-mixing process on $(\Omega, \mathcal A, \mu)$ with $\|\cdot\|_{\mathcal{C}}$ being defined for some semi-norm $\vertiii{\cdot}$ that satisfies Assumption \ref{seminorm_assump}.  Assume that there exists $r_0\geq 1$ such that $\Omega \subset B_{r_0}\subset \mathbb{R}^d$. Then for all $h \leq 1$, $\tau \geq 1$, and $n\geq n_1$ with $n_1$  given in \eqref{nzeroOI}, with probability $\mu$ at least $1 - e^{-\tau}$, there holds
% \begin{align*}
% \|f_{D,h} - f_{P,h}\|_1
% \lesssim    \sqrt{\frac{(\log n)^{2/\gamma}   (\tau + \log (\frac{n r_0}{h}))}{h^d n}}
%         + \frac{(\log n)^{2/\gamma}  (\tau + \log (\frac{n r_0}{h}))}{h^d n}  
% \\
% \phantom{=}        
%      +  \sqrt{\frac{32\tau(\log n)^{2/\gamma}}{n}} +   h^{\beta}.
% \end{align*}
% \end{corollary} 

Our next result shows that when the density function $f$ is compactly supported and bounded, 
an oracle inequality under $L_\infty$-norm can be also derived.
\begin{theorem}\label{OracleInequalityInftyNorm}
Let $K$ be a $d$-dimensional kernel function that satisfies Conditions $(i)$ and $(iii)$ in Assumption \ref{assump_kernel}. 
Let $\mathcal{X} := (X_n)_{n \geq 1}$ be an $X$-valued stationary geometrically (time-reversed) 
$\mathcal{C}$-mixing process on $(\Omega, \mathcal A, \mu)$ with $\|\cdot\|_{\mathcal{C}}$ 
being defined for some semi-norm $\vertiii{\cdot}$ that satisfies Assumption \ref{seminorm_assump}. 
Assume that there exists a constant $r_0\geq 1$ such that $\Omega \subset B_{r_0} \subset \mathbb{R}^d$ and the density function $f$ satisfies $\|f\|_\infty < \infty$. Then for all $0 < h \leq 1$ and $\tau > 0$, there exists an $n_0^* \in \mathbb{N}$ such that for all $n \geq n_0^*$, with probability $\mu$ at least $1 - e^{-\tau}$, there holds
\begin{align*}
\|f_{D,h} - f_{P,h}\|_\infty
\lesssim  \sqrt{\frac{  \|f\|_{\infty} (\tau + \log ( \frac{n r_0 }{h}))(\log n)^{2/\gamma}}{h^d n}}
  + \frac{  K(0) (\tau + \log (\frac{n r_0 }{h}	))(\log n)^{2/\gamma}}{h^d n}.
\end{align*}
Here $n_0^*$ will be given explicitly in the proof.
\end{theorem}

In Theorem \ref{OracleInequalityInftyNorm}, the kernel $K$ is only required to satisfy Conditions $(i)$ and $(iii)$  in Assumption \ref{assump_kernel} whereas the condition that $ \int_0^{\infty} K(r) r^{\beta+d-1} \, \mathrm{d}r < \infty$ for some $\beta>0$ is not needed. This is again due to the compact support assumption of the density function $f$ as stated in Theorem \ref{OracleInequalityInftyNorm}.  

\subsection{Results on Universal Consistency}\label{sec::subsec::universal_consis}

We now present results on the universal consistency property of the kernel density estimator $f_{D,h}$ 
in the sense of $L_1$-norm. A kernel density estimator $f_{D,h}$ is said to be \textit{universally consistent} in the sense of $L_1$-norm if $f_{D,h}$ converges to $f$ almost surely under  $L_1$-norm without any further restrictions on the probability distribution $P$.  
\begin{theorem} \label{ConsistencyL1}
Let $K$ be a $d$-dimensional smoothing kernel that satisfies Conditions $(i)$ and $(iii)$ in Assumption \ref{assump_kernel}.  Let $\mathcal X := (X_n)_{n \geq 1}$ be an $X$-valued stationary geometrically (time-reversed) $\mathcal{C}$-mixing process on $(\Omega, \mathcal A, \mu)$ with $\|\cdot\|_{\mathcal{C}}$ being defined for some semi-norm $\vertiii{\cdot}$ that satisfies Assumption \ref{seminorm_assump}. If 
\begin{align*}
h_n \to 0 \quad \hbox{and} \quad \frac{n h_n^d}{(\log n)^{(2+\gamma)/\gamma}} \to \infty, \quad\hbox{as}\quad n\to \infty,
\end{align*}
then the kernel density estimator $f_{D,h_n}$ is universally consistent in the sense of $L_1$-norm.
\end{theorem}

\subsection{Convergence Rates under $L_1$-Norm}\label{subsec::l1_convergence}
The consistency result in Theorem \ref{ConsistencyL1} is independent of the probability distribution $P$ and is therefore said to be universal. In this subsection, we will show that if certain  tail assumptions on $P$ are applied, convergence rates can be obtained under $L_1$-norm. Here, we consider three different situations, namely, the tail of the probability distribution $P$ has a polynomial decay, exponential decay and disappears, respectively.

\begin{theorem}\label{ConvergenceRatesL1}
Let $K$ be a $d$-dimensional smoothing kernel that satisfies Assumption \ref{assump_kernel}. 
Assume that the density $f$ is $\alpha$-H\"{o}lder continuous with $\alpha \leq \beta$. 
Let $\mathcal X := (X_n)_{n \geq 1}$ be an $X$-valued stationary geometrically (time-reversed) 
$\mathcal{C}$-mixing process on $(\Omega, \mathcal A, \mu)$ with $\|\cdot\|_{\mathcal{C}}$ 
being defined for some semi-norm $\vertiii{\cdot}$ that satisfies Assumption \ref{seminorm_assump}. 
We consider the following cases:
\begin{enumerate}
 \item[(i)] $P \bigl( H_r \bigr) 
    \lesssim  r^{- \eta d}$ for some $\eta>0$ and for all $ r \geq 1$;
 \item[(ii)] $P \bigl( H_r \bigr)
    \lesssim   e^{- a r^\eta}$  for some $a>0$, $\eta>0$ and for all $ r \geq 1$;
 \item[(iii)] $P \bigl( H_{r_0} \bigr) = 0$ for some $r_0 \geq 1$.
\end{enumerate} 
For the above cases, if $n\geq n_0$ with $n_0$ the same as in Theorem \ref{OracleInequalityL1}, and the sequences $h_n$ are of the following forms:
\begin{enumerate}
 \item[(i)] $h_n = \left( \frac{(\log n)^{(2+\gamma)/\gamma}}{n} \right)^{\frac{1+\eta}{(1+\eta)(2\alpha+d)-\alpha}}$;
 \item[(ii)] $h_n = \left( \frac{(\log n)^{(2+\gamma)/\gamma}}{n} \right)^{\frac{1}{2\alpha+d}}
     (\log n)^{- \frac{d}{\gamma} \cdot \frac{1}{2\alpha+d}}$;
 \item[(iii)] $h_n = \left(  (\log n)^{(2+\gamma)/\gamma}/n \right)^{\frac{1}{2\alpha+d}}$;
\end{enumerate}
then with probability $\mu$ at least $1 - \frac{1}{n}$, there holds
\begin{align*}
  \|f_{D,h_n} - f\|_1 \leq \varepsilon_n, 
\end{align*}
where the convergence rates  
\begin{enumerate}
 \item[(i)] $\varepsilon_n \lesssim \left( \frac{(\log n)^{(2+\gamma)/\gamma}}{n} \right)^{\frac{\alpha\eta}{(1+\eta)(2\alpha+d)-\alpha}}$;
 \item[(ii)] $\varepsilon_n \lesssim \left( \frac{(\log n)^{(2+\gamma)/\gamma}}{n} \right)^{\frac{\alpha}{2\alpha+d}}
     (\log n)^{\frac{d}{\gamma} \cdot \frac{\alpha+d}{2\alpha+d}}$;
 \item[(iii)] $\varepsilon_n \lesssim \left( (\log n)^{(2+\gamma)/\gamma}/n\right)^{\frac{\alpha}{2\alpha+d}}$.
\end{enumerate}
\end{theorem}

\subsection{Convergence Rates under $L_\infty$-Norm}\label{subsec::l_infty_convergence}
In Subsection  \ref{oracle_inequalites}, when the density function $f$ is bounded and compactly supported, we establish oracle inequality of $f_{D,h}$ under $L_\infty$-norm. Combining this with the estimate of the deterministic error term in Theorem \ref{ApproximationError} \textit{(ii)} under $L_\infty$-norm, we arrive at the following result that characterizes the convergence of $f_{D,h}$ to $f$ under $L_\infty$-norm.

\begin{theorem}\label{ConvergenceRatesLInfty}
Let $K$ be a $d$-dimensional smoothing kernel that satisfies Conditions $(i)$ and $(iii)$ in Assumption \ref{assump_kernel}. 
Let $\mathcal X := (X_n)_{n \geq 1}$ be an $X$-valued stationary geometrically (time-reversed) 
$\mathcal{C}$-mixing process on $(\Omega, \mathcal A, \mu)$ with $\|\cdot\|_{\mathcal{C}}$ 
being defined for some semi-norm $\vertiii{\cdot}$ that satisfies Assumption \ref{seminorm_assump}. 
Assume that there exists a constant $r_0\geq 1$ such that $\Omega \subset B_{r_0} \subset \mathbb{R}^d$ 
and the density function $f$ is $\alpha$-H\"older continuous with $\alpha\leq \beta$ and $\|f\|_\infty < \infty$. 
Then for all  $n \geq n_0^*$ with $n_0^*$ as in Theorem \ref{OracleInequalityInftyNorm}, by choosing 
\begin{align*}
h_n = \left(  (\log n)^{(2+\gamma)/\gamma}/n\right)^{\frac{1}{2\alpha+d}},
\end{align*}
with probability $\mu$ at least $1 - \frac{1}{n}$, there holds
\begin{align}\label{RatesInfty}
\|f_{D,h_n} - f\|_{\infty} \lesssim  \left(  (\log n)^{(2+\gamma)/\gamma}/n \right)^{\frac{\alpha}{2\alpha+d}}.	
\end{align}	
\end{theorem}

In Theorems \ref{ConvergenceRatesL1} and \ref{ConvergenceRatesLInfty}, 
one needs to ensure that 
$n\geq n_0$ with $n_0$ as in Theorem \ref{OracleInequalityL1} and
$n\geq n_0^*$ with $n_0^*$ as in Theorem \ref{OracleInequalityInftyNorm}, respectively. 
One may also note that due to the involvement of the term $\varphi(h_n)$, 
the numbers $n_0$ and $n_0^*$ depend on the $h_n$. 
However, recalling that for the Triangle kernel,
the Epanechnikov kernel, and the Gaussian kernel,  
we have $\varphi(h_n) \leq \mathcal{O}(h_n^{-1})$, which, 
together with the choices of $h_n$ in Theorems \ref{ConvergenceRatesL1} and \ref{ConvergenceRatesLInfty}, 
implies that $n_0$ and $n_0^*$ are well-defined. 
It should be also remarked that in the scenario where the density function $f$ is compactly supported and bounded,  
the convergence rate of $f_{D,h}$ to $f$ is not only obtainable, but also the same  with that derived under $L_1$-norm. 
This is indeed an interesting observation since convergence under $L_\infty$-norm implies convergence under  $L_1$-norm.

\subsection{Comments and Discussions} \label{CandD}
This section presents some comments on the obtained theoretical results on the consistency and convergence rates of $f_{D,h}$ and compares them with related findings in the literature. 

% To this end, let us first summarize our main theoretical results: Under mild assumptions on the kernel $K$, we establish oracle inequalities for $f_{D,h}$ under various circumstances with respect to $L_1$-norm and $L_\infty$-norm by using capacity-dependent and concentration techniques; Based on these established oracle inequalities, with properly chosen bandwidth, we show that $f_{D,h}$ is universally consistent in the sense of $L_1$-norm; By assuming the regularity of the density function $f$ and the tail behavior of the probability distribution $P$, we also develop convergence rates for $f_{D,h}$ under $L_1$-norm and $L_\infty$-norm.  
 
We highlight that in our analysis the density function $f$ is only assumed to be H\"{o}lder continuous. 
As pointed out in the introduction, in the context of dynamical systems, this seems to be more than a reasonable assumption. 
On the other hand, the consistency, as well as the convergence results obtained in our study, are of type ``with high probability" due to the use of 
the Bernstein-type exponential inequality that takes into account the variance information of the random variables. 
From our analysis and the obtained theoretical results, one can also easily observe the influence of the dependence among observations. 
For instance, from Theorem \ref{ConsistencyL1} we see that with increasing dependence among observations (corresponding to smaller $\gamma$), 
in order to ensure the universal consistency of $f_{D,h_n}$, the decay of $h_n$ (with respect to $n^{-1}$) is required to be faster. 
This is in fact also the case if we look at results on the convergence rates in Theorems \ref{ConvergenceRatesL1} and \ref{ConvergenceRatesLInfty}. 
Moreover, the influence of the dependence among observations is also indicated there. That is, an increase of the dependence among observations 
may slow down the convergence of $f_{D,h}$ in the sense of both $L_1$-norm and $L_\infty$-norm. It is also interesting to note that when $\gamma$ tends to infinity, 
which corresponds to the case where observations can be roughly treated  as independent ones, meaningful convergence rates can be also deduced. It turns out that, up to a logarithmic factor, the established convergence rates \eqref{RatesInfty} under $L_\infty$-norm, namely, 
$\mathcal{O} ( (  (\log n)^{(2+\gamma)/\gamma}/n )^{\alpha/(2\alpha+d)})$, match the optimal rates in the i.i.d.~case, 
see, e.g., \cite{khas1979lower} and \cite{stone1983optimal}.

As mentioned in the introduction, there exist several studies in the literature that address the kernel density estimation problem for dynamical systems. For example, \cite{bosq1995nonparametric} conducted some first studies and showed the point-wise consistency as well as convergence (in expectation) of the kernel density estimator. 
The convergence rates obtained in their study are of the type $\mathcal{O}(n^{-4/(4+2d)})$, which are conducted in terms of the variance of $f_{D,h}$. 
The notion they used for measuring the dependence among observations is $\alpha$-mixing coefficient (see $\text{A}_3$ in \cite{bosq1995nonparametric}). 
Considering the density estimation problem for one-dimensional dynamical systems, \cite{prieur2001density} presented some studies 
on the kernel density estimator $f_{D,h}$ by developing a central limit theorem and apply it to bound the variance of the estimator.  
Further some studies on the kernel density estimation of the invariant Lebesgue density for dynamical systems were conducted in \cite{blanke2003modelization}. 
By considering both dynamical noise and observational noise, point-wise convergence of the estimator $f_{D,h}$ in expectation was established, 
i.e., the convergence of $\mathbb{E}f_{D,h}(x)-f(x)$ for any $x\in\mathbb{R}^d$. 
Note further that these results rely on the second-order smoothness and boundedness of $f$. 
Therefore, the second-order smoothness assumption on the density function together with the point-wise convergence 
in expectation makes it different from our work. In particular, under the additional assumption on the tail of the noise distribution, 
the convergence of $\mathbb{E}(f_{D,h}(x)-f(x))^2$ for any fixed $x\in\mathbb{R}^d$ is of the order $\mathcal{O}(n^{-2/(2+\beta d)})$ with $\beta\geq 1$. Concerning the convergence of $f_{D,h}$ in a dynamical system setup, \cite{maume2006exponential} also presented some interesting studies which in some sense also motivated  our work here. 
By using also the $\mathcal{C}$-mixing concept as adopted in our study to measure the dependence among observations from dynamical systems, 
she presented the point-wise convergence of $f_{D,h}$ with the help of Hoeffding-type exponential inequality (see Proposition 3.1 in \cite{maume2006exponential}). 
The assumption applied on $f$ is that it is bounded from below and also $\alpha$-H\"{o}lder continuous (more precisely, $f$ is assumed to be \textit{$\alpha$-regular}, 
see Assumption 2.3 in  \cite{maume2006exponential}). 
Hence, from the above discussions, we suggest that the work we present in this study is essentially different from that in \cite{maume2006exponential}.

\section{Bandwidth Selection and Simulation Studies}\label{sec::bandwidth_selection}
This section discusses the model selection problem of the kernel density estimator \eqref{KDE_formal} by performing numerical simulation studies. In the context of kernel density estimation, model selection is mainly referred to the choice of the smoothing kernel $K$ and the selection of the kernel bandwidth $h$, which are of crucial importance for the practical implementation of the data-driven density estimator. According to our experimental experience and also the empirical observations reported in \cite{maume2006exponential}, it seems that the choice of the kernel or the noise does not have a significant influence on the performance of the estimator. Therefore, our emphasis will be placed on the bandwidth selection problem in our simulation studies. 

\subsection{Several Bandwidth Selectors}
In the literature of kernel density estimation, various bandwidth selectors have been proposed, several typical examples of which have been alluded to in the introduction. When turning to the case with dependent observations, the bandwidth selection problem has been also drawing much attention, see e.g., \cite{hart1990data,chu1991comparison,hall1995bandwidth,yao1998cross}. Among existing bandwidth selectors, probably the most frequently employed ones are based on the cross-validation ideas. For cross-validation bandwidth selectors, one tries to minimize the integrated squared error (ISE) of the empirical estimator $f_{D,h}$ where 
\begin{align*}
\text{ISE}(h):=\int (f_{D,h}-f)^2  = \int f^2_{D,h} - 2\int f_{D,h} \cdot \int f + \int f^2.
\end{align*}  
Note that on the right-hand side of the above equality, the last term $\int f^2$ is independent of $h$ and so the minimization of $\text{ISE}(h)$ is equivalent to minimize $$\int f^2_{D,h} - 2\int f_{D,h} \cdot \int f.$$ It is shown that with i.i.d observations, an unbiased estimator of the above quantity, which is termed as least squares cross-validation (LSCV), is given as follows:
\begin{align}\label{LSCV}
\text{LSCV}(h):= \int f_{D,h}^2 - \frac{2}{n}\sum_{i=1}^n \hat{f}_{-i,h}(x_i),
\end{align} 
where the leave-one-out density estimator $\hat{f}_{-i,h}$ is defined as $$\hat{f}_{-i,h}(x):=\frac{1}{n-1}\sum_{j\neq i}^n K_h(x-x_j).$$ When the observations are dependent, it is shown that cross-validation can produce much under-smoothed estimates, see e.g., \cite{hart1986kernel,hart1990data}. Observing this, \cite{hart1990data} proposed the modified least squares cross-validation (MLSCV), which is defined as follows 	
\begin{align}\label{MLSVC}
\text{MLSCV}(h):= \int f_{D,h}^2 - \frac{2}{n}\sum_{i=1}^n \hat{f}_{-i,h,l_n}(x_i), 
\end{align} 
where $l_n$ is set to $1$ or $2$ as suggested in \cite{hart1990data} and 	
\begin{align*}
\hat{f}_{-i,h,l_n}(x) := \frac{1}{\#\{j:|j-i|>l_n\}}\sum_{|j-i|>l_n}K_h(x-x_j).
\end{align*}  
The underlying intuition of proposing MLSCV is that when estimating the density of a fixed point, ignoring observations in the vicinity of this point may be help in reducing the influence of dependence among observations. However, when turning to the $L_1$ point of view, the above bandwidth selectors may not work well due to the use of the least squares criterion. Alternatively, \cite{devroye1989double} proposed the double kernel bandwidth selector that minimizes the following quantity
\begin{align}\label{DKM}
\text{DKM}(h):= \int |f_{D,h,K}-f_{D,h,L}|,
\end{align}   
where $f_{D,h,K}$ and $f_{D,h,L}$ are kernel density estimators based on the kernels $K$ and $L$, respectively. Some rigorous theoretical treatments on the effectiveness of the above bandwidth selector were made in \cite{devroye1989double}.

Our purpose in simulation studies is to conduct empirical comparisons among the above bandwidth selectors in the dynamical system context instead of proposing new approaches.

\subsection{Experimental Setup}
In our experiments, observations $x_1,\cdots, x_n$ are generated from the  following model\footnote{Note that here the observational noise is assumed for the considered dynamical system \eqref{experiment::model}, which differs from \eqref{model_dynamical_systems} and can be a more realistic setup from an empirical and experimental viewpoint. In fact, it is observed also in \cite{maume2006exponential} that the influence of   low SNR noise is not obvious in density estimation. We therefore adopt this setup in our experiments. All the observations reported in this experimental section apply to the noiseless case \eqref{model_dynamical_systems}.}
\begin{align}\label{experiment::model}
\begin{cases}
\tilde{x}_i = T^i(x_0),\\
x_i = \tilde{x}	_i + \varepsilon_i,
\end{cases}
\,\, i=1,\cdots,n,
\end{align} 
where $\varepsilon_i \sim \mathcal{N}(0,\sigma^2)$, $\sigma$ is set to $0.01$ and the initial state $x_0$ is randomly generated based on the density $f$. For the map $T$ in \eqref{experiment::model}, we choose Logistic map in Example \ref{logistic_map} and Gauss map in Example \ref{gauss_map}. We vary the sample size among $\{5\times10^2,10^3,5\times10^3,10^4\}$, implement bandwidth selection procedures over $2	0$ replications and select the bandwidth from a grid of values in the interval $[h_L,h_U]$ with $100$ equispaced points. Here, $h_L$ is set as the minimum distance between consecutive points $x_i,\,i=1,\cdots,n$ \citep{devroye1997nonasymptotic}, while $h_U$ is chosen according to the \textit{maximal smoothing principle} proposed in \cite{terrell1990maximal}. Throughout our experiments, we use the Gaussian kernel for the kernel density estimators.

In our experiments, we conduct comparisons among the above-mentioned bandwidth selectors which are, respectively, denoted as follows:
\begin{itemize}
\item LSCV: the least squares cross-validation given in \eqref{LSCV};  
\item MLSCV-1: the modified least squares cross-validation in \eqref{MLSVC} with $l_n=1$;
\item MLSCV-2: the modified least squares cross-validation in \eqref{MLSVC} with $l_n=2$;
\item DKM: the double kernel method defined in \eqref{DKM} where the two kernels used here are the Epanechnikov kernel and the Triangle kernel, respectively.
\end{itemize}
In the experiments, due to the known density functions for Logistic map and Gauss map, and in accordance with our previous analysis from the $L_1$ point of view, the criterion of comparing different selected bandwidths is the following absolute mean error (AME):
$$\text{AME}(h)= \frac{1}{m}\sum_{i=1}^m|f_{D,h}(u_i)-f(u_i)|,$$
where $u_1, \cdots,u_m$ are $m$ equispaced points in the interval $[0,1]$ and $m$ is set to $10000$. We also compare the selected bandwidth with the one that has the minimum absolute mean error which serves as a \textbf{\textit{baseline}} method in our experiments. 
 
\subsection{Simulation Results and Observations}
The AMEs of the above bandwidth selectors for Logistic map in Example \ref{logistic_map} and Gauss map in Example \ref{gauss_map} over $20$ replications are averaged and recorded in Tables \ref{logistic_map_table} and \ref{gauss_map_table} below.
 
In Figs.\,\ref{logistic_map_figure}  and \ref{gauss_map_figure}, we also plot the kernel density estimators for Logistic map in Example \ref{logistic_map} and Gauss map in Example \ref{gauss_map} with different bandwidths and their true density functions with different sample sizes. The sample size of each panel, in Figs.\,\ref{logistic_map_figure}  and \ref{gauss_map_figure},  from up to bottom, is $10^3$, $10^4$ and $10^5$, respectively. In each panel, the densely dashed black curve represents the true density, the dotted blue curve is the estimated density function with the bandwidth selected by the baseline method while the solid red curve stands for the estimated density with the bandwidth selected by the double kernel method. All density functions in Figs.\,\ref{logistic_map_figure}  and \ref{gauss_map_figure} are plotted with $100$ equispaced points in the interval $(0,1)$.
 
\begin{table}[h]
 \setlength{\tabcolsep}{18pt}
   \centering
     \captionsetup{justification=centering}
\caption{The AMEs of Different Bandwidth Selectors for Logistic Map in Example \ref{logistic_map}}
\label{logistic_map_table}
\begin{tabular}{@{}c cccc|c@{}}
\toprule
  \text{sample size}& \text{LSCV} & \text{MLSCV-1} & \text{MLSCV-2} &\text{DKM}  & \text{Baseline} \\ \midrule
 $5\times 10^2$ & .3372 & .3369 & .3372 &  .3117  & .3013 \\
 $1\times 10^3$ & .2994 & .2994 & .2994 &  .2804  & .2770 \\
 $5\times 10^3$ & .2422 & .2422 & .2422 &  .2340  &  .2326 \\
 $1\times 10^4$ & .2235 & .2235 & .2235 &  .2220  & .2192 \\ \bottomrule
\end{tabular}
\end{table}

\begin{table}[h]
 \setlength{\tabcolsep}{18pt}
   \centering
     \captionsetup{justification=centering}
\caption{The AMEs of Different Bandwidth Selectors for Gauss Map in Example \ref{gauss_map}}
\label{gauss_map_table}
\begin{tabular}{@{}c cccc|c@{}}
\toprule
  \text{sample size}& \text{LSCV} & \text{MLSCV-1} & \text{MLSCV-2} &\text{DKM}  & \text{Baseline} \\ \midrule
 $5\times 10^2$ & .1027 & .1026 & .1059 &  .1181 & .0941 \\
 $1\times 10^3$ & .0925  & .0933  & .0926  &  .0925   & .0878  \\
 $5\times 10^3$ & .0626  & .0626  & .0626  &  .0586   &  .0585  \\
 $1\times 10^4$ & .0454  & .0454  & .0454  &  .0440   & .0439  \\ \bottomrule
\end{tabular}
\end{table}

\begin{figure}
\center
 \hspace{-6cm}\begin{minipage}[b]{0.3\textwidth}
            \centering
\begin{tikzpicture}

\begin{axis}[%
width=4.2in,
height=2.0in,
at={(0.758in,0.474in)},
scale only axis,
separate axis lines,
every outer x axis line/.append style={black},
every x tick label/.append style={font=\color{black}},
xmin=0.000,
xmax=1.000,
xlabel ={$x$},
ylabel = {$f_{D,h}(x)$},
every outer y axis line/.append style={black},
every y tick label/.append style={font=\color{black}},
ymin=0.000,
ymax=12.000,
axis background/.style={fill=white}
]
\addplot [ color=black, thick, densely dashed, forget plot]
  table[row sep=crcr]{%
0.001	10.071\\
0.011	3.042\\
0.021	2.213\\
0.031	1.831\\
0.041	1.600\\
0.051	1.442\\
0.061	1.326\\
0.071	1.235\\
0.082	1.163\\
0.092	1.103\\
0.102	1.053\\
0.112	1.010\\
0.122	0.973\\
0.132	0.941\\
0.142	0.912\\
0.152	0.886\\
0.162	0.864\\
0.172	0.843\\
0.182	0.824\\
0.192	0.808\\
0.202	0.792\\
0.212	0.778\\
0.223	0.765\\
0.233	0.753\\
0.243	0.742\\
0.253	0.732\\
0.263	0.723\\
0.273	0.715\\
0.283	0.707\\
0.293	0.699\\
0.303	0.693\\
0.313	0.686\\
0.323	0.681\\
0.333	0.675\\
0.343	0.670\\
0.353	0.666\\
0.364	0.662\\
0.374	0.658\\
0.384	0.655\\
0.394	0.651\\
0.404	0.649\\
0.414	0.646\\
0.424	0.644\\
0.434	0.642\\
0.444	0.641\\
0.454	0.639\\
0.464	0.638\\
0.474	0.637\\
0.484	0.637\\
0.494	0.637\\
0.505	0.637\\
0.515	0.637\\
0.525	0.637\\
0.535	0.638\\
0.545	0.639\\
0.555	0.640\\
0.565	0.642\\
0.575	0.644\\
0.585	0.646\\
0.595	0.648\\
0.605	0.651\\
0.615	0.654\\
0.625	0.658\\
0.635	0.661\\
0.646	0.665\\
0.656	0.670\\
0.666	0.675\\
0.676	0.680\\
0.686	0.686\\
0.696	0.692\\
0.706	0.699\\
0.716	0.706\\
0.726	0.714\\
0.736	0.722\\
0.746	0.731\\
0.756	0.741\\
0.766	0.752\\
0.776	0.764\\
0.787	0.777\\
0.797	0.791\\
0.807	0.806\\
0.817	0.823\\
0.827	0.841\\
0.837	0.861\\
0.847	0.884\\
0.857	0.909\\
0.867	0.938\\
0.877	0.970\\
0.887	1.006\\
0.897	1.049\\
0.907	1.098\\
0.917	1.157\\
0.928	1.228\\
0.938	1.316\\
0.948	1.429\\
0.958	1.582\\
0.968	1.803\\
0.978	2.163\\
0.988	2.915\\
0.998	7.125\\
};
\addplot [  color=red, solid, forget plot]
  table[row sep=crcr]{%
0.001	1.789\\
0.011	2.032\\
0.021	2.080\\
0.031	1.961\\
0.041	1.757\\
0.051	1.551\\
0.061	1.385\\
0.071	1.258\\
0.082	1.155\\
0.092	1.072\\
0.102	1.008\\
0.112	0.958\\
0.122	0.908\\
0.132	0.849\\
0.142	0.784\\
0.152	0.731\\
0.162	0.701\\
0.172	0.697\\
0.182	0.711\\
0.192	0.733\\
0.202	0.755\\
0.212	0.775\\
0.223	0.793\\
0.233	0.814\\
0.243	0.839\\
0.253	0.860\\
0.263	0.871\\
0.273	0.868\\
0.283	0.846\\
0.293	0.804\\
0.303	0.750\\
0.313	0.702\\
0.323	0.683\\
0.333	0.701\\
0.343	0.746\\
0.353	0.787\\
0.364	0.794\\
0.374	0.757\\
0.384	0.693\\
0.394	0.629\\
0.404	0.586\\
0.414	0.569\\
0.424	0.571\\
0.434	0.581\\
0.444	0.590\\
0.454	0.593\\
0.464	0.587\\
0.474	0.578\\
0.484	0.574\\
0.494	0.583\\
0.505	0.604\\
0.515	0.632\\
0.525	0.662\\
0.535	0.687\\
0.545	0.699\\
0.555	0.692\\
0.565	0.668\\
0.575	0.636\\
0.585	0.609\\
0.595	0.592\\
0.605	0.584\\
0.615	0.582\\
0.625	0.590\\
0.635	0.611\\
0.646	0.646\\
0.656	0.687\\
0.666	0.726\\
0.676	0.756\\
0.686	0.773\\
0.696	0.774\\
0.706	0.761\\
0.716	0.740\\
0.726	0.721\\
0.736	0.711\\
0.746	0.714\\
0.756	0.729\\
0.766	0.753\\
0.776	0.780\\
0.787	0.803\\
0.797	0.823\\
0.807	0.849\\
0.817	0.890\\
0.827	0.945\\
0.837	1.001\\
0.847	1.039\\
0.857	1.053\\
0.867	1.052\\
0.877	1.058\\
0.887	1.091\\
0.897	1.158\\
0.907	1.243\\
0.917	1.325\\
0.928	1.394\\
0.938	1.462\\
0.948	1.551\\
0.958	1.673\\
0.968	1.809\\
0.978	1.911\\
0.988	1.909\\
0.998	1.744\\
};
\addplot [dotted, thick, color=blue, forget plot]
  table[row sep=crcr]{%
0.001	2.039\\
0.011	2.389\\
0.021	2.373\\
0.031	2.057\\
0.041	1.681\\
0.051	1.441\\
0.061	1.336\\
0.071	1.256\\
0.082	1.142\\
0.092	1.029\\
0.102	0.971\\
0.112	0.963\\
0.122	0.947\\
0.132	0.875\\
0.142	0.761\\
0.152	0.677\\
0.162	0.660\\
0.172	0.677\\
0.182	0.701\\
0.192	0.735\\
0.202	0.769\\
0.212	0.782\\
0.223	0.780\\
0.233	0.797\\
0.243	0.846\\
0.253	0.882\\
0.263	0.885\\
0.273	0.886\\
0.283	0.884\\
0.293	0.841\\
0.303	0.749\\
0.313	0.650\\
0.323	0.605\\
0.333	0.644\\
0.343	0.746\\
0.353	0.855\\
0.364	0.894\\
0.374	0.813\\
0.384	0.668\\
0.394	0.572\\
0.404	0.548\\
0.414	0.549\\
0.424	0.558\\
0.434	0.582\\
0.444	0.605\\
0.454	0.612\\
0.464	0.600\\
0.474	0.569\\
0.484	0.543\\
0.494	0.559\\
0.505	0.602\\
0.515	0.632\\
0.525	0.659\\
0.535	0.702\\
0.545	0.736\\
0.555	0.729\\
0.565	0.681\\
0.575	0.618\\
0.585	0.581\\
0.595	0.585\\
0.605	0.590\\
0.615	0.572\\
0.625	0.562\\
0.635	0.586\\
0.646	0.637\\
0.656	0.696\\
0.666	0.740\\
0.676	0.769\\
0.686	0.796\\
0.696	0.804\\
0.706	0.775\\
0.716	0.729\\
0.726	0.703\\
0.736	0.700\\
0.746	0.696\\
0.756	0.705\\
0.766	0.748\\
0.776	0.799\\
0.787	0.820\\
0.797	0.816\\
0.807	0.817\\
0.817	0.853\\
0.827	0.936\\
0.837	1.036\\
0.847	1.096\\
0.857	1.089\\
0.867	1.044\\
0.877	1.003\\
0.887	1.017\\
0.897	1.117\\
0.907	1.267\\
0.917	1.382\\
0.928	1.414\\
0.938	1.414\\
0.948	1.475\\
0.958	1.623\\
0.968	1.821\\
0.978	2.034\\
0.988	2.182\\
0.998	2.072\\
};
\end{axis}
\end{tikzpicture}%
\end{minipage}%\	
\\
\hspace{-6cm}\begin{minipage}[b]{0.3\textwidth}
                   \centering
\begin{tikzpicture}

\begin{axis}[%
width=4.2in,
height=2.0in,
at={(0.758in,0.474in)},
scale only axis,
separate axis lines,
every outer x axis line/.append style={black},
every x tick label/.append style={font=\color{black}},
xmin=0.000,
xmax=1.000,
xlabel ={$x$},
ylabel = {$f_{D,h}(x)$},
every outer y axis line/.append style={black},
every y tick label/.append style={font=\color{black}},
ymin=0.000,
ymax=12.000,
axis background/.style={fill=white}
]
\addplot [ color=black, thick, densely dashed, forget plot]
  table[row sep=crcr]{%
0.001	10.071\\
0.011	3.042\\
0.021	2.213\\
0.031	1.831\\
0.041	1.600\\
0.051	1.442\\
0.061	1.326\\
0.071	1.235\\
0.082	1.163\\
0.092	1.103\\
0.102	1.053\\
0.112	1.010\\
0.122	0.973\\
0.132	0.941\\
0.142	0.912\\
0.152	0.886\\
0.162	0.864\\
0.172	0.843\\
0.182	0.824\\
0.192	0.808\\
0.202	0.792\\
0.212	0.778\\
0.223	0.765\\
0.233	0.753\\
0.243	0.742\\
0.253	0.732\\
0.263	0.723\\
0.273	0.715\\
0.283	0.707\\
0.293	0.699\\
0.303	0.693\\
0.313	0.686\\
0.323	0.681\\
0.333	0.675\\
0.343	0.670\\
0.353	0.666\\
0.364	0.662\\
0.374	0.658\\
0.384	0.655\\
0.394	0.651\\
0.404	0.649\\
0.414	0.646\\
0.424	0.644\\
0.434	0.642\\
0.444	0.641\\
0.454	0.639\\
0.464	0.638\\
0.474	0.637\\
0.484	0.637\\
0.494	0.637\\
0.505	0.637\\
0.515	0.637\\
0.525	0.637\\
0.535	0.638\\
0.545	0.639\\
0.555	0.640\\
0.565	0.642\\
0.575	0.644\\
0.585	0.646\\
0.595	0.648\\
0.605	0.651\\
0.615	0.654\\
0.625	0.658\\
0.635	0.661\\
0.646	0.665\\
0.656	0.670\\
0.666	0.675\\
0.676	0.680\\
0.686	0.686\\
0.696	0.692\\
0.706	0.699\\
0.716	0.706\\
0.726	0.714\\
0.736	0.722\\
0.746	0.731\\
0.756	0.741\\
0.766	0.752\\
0.776	0.764\\
0.787	0.777\\
0.797	0.791\\
0.807	0.806\\
0.817	0.823\\
0.827	0.841\\
0.837	0.861\\
0.847	0.884\\
0.857	0.909\\
0.867	0.938\\
0.877	0.970\\
0.887	1.006\\
0.897	1.049\\
0.907	1.098\\
0.917	1.157\\
0.928	1.228\\
0.938	1.316\\
0.948	1.429\\
0.958	1.582\\
0.968	1.803\\
0.978	2.163\\
0.988	2.915\\
0.998	7.125\\
};
\addplot [  color=red, solid, forget plot]
  table[row sep=crcr]{%
0.001	2.356\\
0.011	2.681\\
0.021	2.411\\
0.031	2.019\\
0.041	1.726\\
0.051	1.507\\
0.061	1.363\\
0.071	1.303\\
0.082	1.215\\
0.092	1.104\\
0.102	1.039\\
0.112	0.990\\
0.122	0.946\\
0.132	0.924\\
0.142	0.899\\
0.152	0.868\\
0.162	0.833\\
0.172	0.804\\
0.182	0.799\\
0.192	0.820\\
0.202	0.835\\
0.212	0.813\\
0.223	0.780\\
0.233	0.776\\
0.243	0.769\\
0.253	0.745\\
0.263	0.715\\
0.273	0.675\\
0.283	0.664\\
0.293	0.694\\
0.303	0.737\\
0.313	0.747\\
0.323	0.700\\
0.333	0.652\\
0.343	0.648\\
0.353	0.669\\
0.364	0.668\\
0.374	0.678\\
0.384	0.713\\
0.394	0.708\\
0.404	0.667\\
0.414	0.625\\
0.424	0.610\\
0.434	0.629\\
0.444	0.649\\
0.454	0.650\\
0.464	0.633\\
0.474	0.620\\
0.484	0.614\\
0.494	0.622\\
0.505	0.636\\
0.515	0.650\\
0.525	0.677\\
0.535	0.695\\
0.545	0.669\\
0.555	0.621\\
0.565	0.601\\
0.575	0.596\\
0.585	0.591\\
0.595	0.604\\
0.605	0.641\\
0.615	0.652\\
0.625	0.640\\
0.635	0.676\\
0.646	0.731\\
0.656	0.728\\
0.666	0.700\\
0.676	0.712\\
0.686	0.720\\
0.696	0.685\\
0.706	0.659\\
0.716	0.696\\
0.726	0.745\\
0.736	0.762\\
0.746	0.755\\
0.756	0.714\\
0.766	0.698\\
0.776	0.736\\
0.787	0.791\\
0.797	0.806\\
0.807	0.781\\
0.817	0.761\\
0.827	0.808\\
0.837	0.888\\
0.847	0.944\\
0.857	0.974\\
0.867	0.979\\
0.877	1.006\\
0.887	1.037\\
0.897	1.067\\
0.907	1.098\\
0.917	1.170\\
0.928	1.304\\
0.938	1.402\\
0.948	1.468\\
0.958	1.632\\
0.968	1.944\\
0.978	2.397\\
0.988	2.777\\
0.998	2.488\\
};
\addplot [dotted, thick, color=blue, forget plot]
  table[row sep=crcr]{%
0.001	2.579\\
0.011	2.855\\
0.021	2.412\\
0.031	1.985\\
0.041	1.706\\
0.051	1.501\\
0.061	1.317\\
0.071	1.330\\
0.082	1.224\\
0.092	1.081\\
0.102	1.039\\
0.112	0.992\\
0.122	0.934\\
0.132	0.929\\
0.142	0.901\\
0.152	0.867\\
0.162	0.836\\
0.172	0.795\\
0.182	0.791\\
0.192	0.821\\
0.202	0.849\\
0.212	0.819\\
0.223	0.763\\
0.233	0.781\\
0.243	0.774\\
0.253	0.744\\
0.263	0.722\\
0.273	0.666\\
0.283	0.648\\
0.293	0.693\\
0.303	0.744\\
0.313	0.769\\
0.323	0.699\\
0.333	0.637\\
0.343	0.639\\
0.353	0.682\\
0.364	0.665\\
0.374	0.661\\
0.384	0.734\\
0.394	0.715\\
0.404	0.668\\
0.414	0.617\\
0.424	0.598\\
0.434	0.632\\
0.444	0.653\\
0.454	0.660\\
0.464	0.627\\
0.474	0.620\\
0.484	0.609\\
0.494	0.619\\
0.505	0.639\\
0.515	0.644\\
0.525	0.678\\
0.535	0.710\\
0.545	0.679\\
0.555	0.606\\
0.565	0.601\\
0.575	0.595\\
0.585	0.589\\
0.595	0.589\\
0.605	0.653\\
0.615	0.665\\
0.625	0.618\\
0.635	0.668\\
0.646	0.755\\
0.656	0.739\\
0.666	0.675\\
0.676	0.719\\
0.686	0.736\\
0.696	0.684\\
0.706	0.634\\
0.716	0.695\\
0.726	0.760\\
0.736	0.758\\
0.746	0.776\\
0.756	0.701\\
0.766	0.682\\
0.776	0.729\\
0.787	0.806\\
0.797	0.818\\
0.807	0.783\\
0.817	0.736\\
0.827	0.799\\
0.837	0.901\\
0.847	0.945\\
0.857	0.992\\
0.867	0.961\\
0.877	1.011\\
0.887	1.036\\
0.897	1.069\\
0.907	1.093\\
0.917	1.137\\
0.928	1.324\\
0.938	1.421\\
0.948	1.434\\
0.958	1.592\\
0.968	1.910\\
0.978	2.364\\
0.988	2.993\\
0.998	2.722\\
};
\end{axis}
\end{tikzpicture}%
\end{minipage}%\
\\
\hspace{-6cm}\begin{minipage}[b]{0.3\textwidth}
                          \centering
\begin{tikzpicture}

\begin{axis}[%
width=4.2in,
height=2.0in,
at={(2.578in,1.098in)},
scale only axis,
separate axis lines,
every outer x axis line/.append style={black},
every x tick label/.append style={font=\color{black}},
xmin=0.000,
xmax=1.000,
xlabel ={$x$},
ylabel = {$f_{D,h}(x)$},
every outer y axis line/.append style={black},
every y tick label/.append style={font=\color{black}},
ymin=0.000,
ymax=12.000,
axis background/.style={fill=white}
]
\addplot [ color=black, thick, densely dashed, forget plot]
  table[row sep=crcr]{%
0.001	10.071\\
0.011	3.042\\
0.021	2.213\\
0.031	1.831\\
0.041	1.600\\
0.051	1.442\\
0.061	1.326\\
0.071	1.235\\
0.082	1.163\\
0.092	1.103\\
0.102	1.053\\
0.112	1.010\\
0.122	0.973\\
0.132	0.941\\
0.142	0.912\\
0.152	0.886\\
0.162	0.864\\
0.172	0.843\\
0.182	0.824\\
0.192	0.808\\
0.202	0.792\\
0.212	0.778\\
0.223	0.765\\
0.233	0.753\\
0.243	0.742\\
0.253	0.732\\
0.263	0.723\\
0.273	0.715\\
0.283	0.707\\
0.293	0.699\\
0.303	0.693\\
0.313	0.686\\
0.323	0.681\\
0.333	0.675\\
0.343	0.670\\
0.353	0.666\\
0.364	0.662\\
0.374	0.658\\
0.384	0.655\\
0.394	0.651\\
0.404	0.649\\
0.414	0.646\\
0.424	0.644\\
0.434	0.642\\
0.444	0.641\\
0.454	0.639\\
0.464	0.638\\
0.474	0.637\\
0.484	0.637\\
0.494	0.637\\
0.505	0.637\\
0.515	0.637\\
0.525	0.637\\
0.535	0.638\\
0.545	0.639\\
0.555	0.640\\
0.565	0.642\\
0.575	0.644\\
0.585	0.646\\
0.595	0.648\\
0.605	0.651\\
0.615	0.654\\
0.625	0.658\\
0.635	0.661\\
0.646	0.665\\
0.656	0.670\\
0.666	0.675\\
0.676	0.680\\
0.686	0.686\\
0.696	0.692\\
0.706	0.699\\
0.716	0.706\\
0.726	0.714\\
0.736	0.722\\
0.746	0.731\\
0.756	0.741\\
0.766	0.752\\
0.776	0.764\\
0.787	0.777\\
0.797	0.791\\
0.807	0.806\\
0.817	0.823\\
0.827	0.841\\
0.837	0.861\\
0.847	0.884\\
0.857	0.909\\
0.867	0.938\\
0.877	0.970\\
0.887	1.006\\
0.897	1.049\\
0.907	1.098\\
0.917	1.157\\
0.928	1.228\\
0.938	1.316\\
0.948	1.429\\
0.958	1.582\\
0.968	1.803\\
0.978	2.163\\
0.988	2.915\\
0.998	7.125\\
};
\addplot [  color=red, solid, forget plot]
  table[row sep=crcr]{%
0.001	2.550\\
0.011	2.897\\
0.021	2.533\\
0.031	2.037\\
0.041	1.693\\
0.051	1.462\\
0.061	1.306\\
0.071	1.233\\
0.082	1.184\\
0.092	1.128\\
0.102	1.071\\
0.112	1.009\\
0.122	0.974\\
0.132	0.944\\
0.142	0.916\\
0.152	0.893\\
0.162	0.886\\
0.172	0.869\\
0.182	0.840\\
0.192	0.801\\
0.202	0.799\\
0.212	0.785\\
0.223	0.753\\
0.233	0.741\\
0.243	0.722\\
0.253	0.712\\
0.263	0.730\\
0.273	0.722\\
0.283	0.695\\
0.293	0.694\\
0.303	0.693\\
0.313	0.690\\
0.323	0.688\\
0.333	0.674\\
0.343	0.655\\
0.353	0.653\\
0.364	0.656\\
0.374	0.668\\
0.384	0.681\\
0.394	0.670\\
0.404	0.647\\
0.414	0.646\\
0.424	0.652\\
0.434	0.637\\
0.444	0.633\\
0.454	0.643\\
0.464	0.642\\
0.474	0.647\\
0.484	0.652\\
0.494	0.645\\
0.505	0.647\\
0.515	0.643\\
0.525	0.625\\
0.535	0.614\\
0.545	0.631\\
0.555	0.637\\
0.565	0.638\\
0.575	0.659\\
0.585	0.670\\
0.595	0.647\\
0.605	0.623\\
0.615	0.642\\
0.625	0.661\\
0.635	0.664\\
0.646	0.659\\
0.656	0.662\\
0.666	0.676\\
0.676	0.696\\
0.686	0.701\\
0.696	0.704\\
0.706	0.688\\
0.716	0.669\\
0.726	0.681\\
0.736	0.695\\
0.746	0.709\\
0.756	0.724\\
0.766	0.739\\
0.776	0.767\\
0.787	0.789\\
0.797	0.789\\
0.807	0.796\\
0.817	0.804\\
0.827	0.809\\
0.837	0.857\\
0.847	0.910\\
0.857	0.927\\
0.867	0.932\\
0.877	0.973\\
0.887	1.014\\
0.897	1.055\\
0.907	1.110\\
0.917	1.163\\
0.928	1.225\\
0.938	1.317\\
0.948	1.459\\
0.958	1.658\\
0.968	1.973\\
0.978	2.462\\
0.988	2.862\\
0.998	2.603\\
};
\addplot [dotted, thick, color=blue, forget plot]
  table[row sep=crcr]{%
0.001	2.746\\
0.011	3.055\\
0.021	2.536\\
0.031	1.996\\
0.041	1.675\\
0.051	1.450\\
0.061	1.285\\
0.071	1.232\\
0.082	1.188\\
0.092	1.121\\
0.102	1.082\\
0.112	0.994\\
0.122	0.980\\
0.132	0.937\\
0.142	0.922\\
0.152	0.883\\
0.162	0.890\\
0.172	0.867\\
0.182	0.852\\
0.192	0.783\\
0.202	0.809\\
0.212	0.790\\
0.223	0.743\\
0.233	0.747\\
0.243	0.720\\
0.253	0.701\\
0.263	0.738\\
0.273	0.727\\
0.283	0.685\\
0.293	0.698\\
0.303	0.691\\
0.313	0.694\\
0.323	0.686\\
0.333	0.680\\
0.343	0.649\\
0.353	0.653\\
0.364	0.654\\
0.374	0.667\\
0.384	0.685\\
0.394	0.676\\
0.404	0.640\\
0.414	0.643\\
0.424	0.664\\
0.434	0.630\\
0.444	0.630\\
0.454	0.648\\
0.464	0.637\\
0.474	0.649\\
0.484	0.655\\
0.494	0.645\\
0.505	0.647\\
0.515	0.643\\
0.525	0.628\\
0.535	0.605\\
0.545	0.636\\
0.555	0.641\\
0.565	0.627\\
0.575	0.662\\
0.585	0.678\\
0.595	0.652\\
0.605	0.609\\
0.615	0.643\\
0.625	0.667\\
0.635	0.666\\
0.646	0.656\\
0.656	0.660\\
0.666	0.674\\
0.676	0.703\\
0.686	0.698\\
0.696	0.709\\
0.706	0.692\\
0.716	0.655\\
0.726	0.686\\
0.736	0.696\\
0.746	0.707\\
0.756	0.726\\
0.766	0.734\\
0.776	0.769\\
0.787	0.795\\
0.797	0.788\\
0.807	0.792\\
0.817	0.813\\
0.827	0.794\\
0.837	0.856\\
0.847	0.921\\
0.857	0.932\\
0.867	0.915\\
0.877	0.979\\
0.887	1.015\\
0.897	1.053\\
0.907	1.111\\
0.917	1.161\\
0.928	1.220\\
0.938	1.305\\
0.948	1.453\\
0.958	1.638\\
0.968	1.925\\
0.978	2.467\\
0.988	2.994\\
0.998	2.818\\
};
\end{axis}
\end{tikzpicture}%
\end{minipage}%\
\caption{Plots of the kernel density estimators $f_{D,h}$ for Logistic map in Example \ref{logistic_map} with different bandwidths and its true density with different sample sizes. The sample size of each panel, from up to bottom, is $10^3$, $10^4$ and $10^5$, respectively. In each panel, the dashed black curve represents the true density of Logistic map, the dotted blue curve is the estimated density of Logistic map with the bandwidth selected by the baseline method while the solid red curve stands for the estimated density of Logistic map with the bandwidth selected by the double kernel method. All curves are plotted with $100$ equispaced points in the interval $(0,1)$.}\label{logistic_map_figure}  
\end{figure}  

\begin{figure}
\center
 \hspace{-6cm}\begin{minipage}[b]{0.3\textwidth}
            \centering
\begin{tikzpicture}

\begin{axis}[%
width=4.2in,
height=2.0in,
at={(0.758in,0.474in)},
scale only axis,
separate axis lines,
every outer x axis line/.append style={black},
every x tick label/.append style={font=\color{black}},
xmin=0.000,
xmax=1.000,
xlabel ={$x$},
ylabel = {$f_{D,h}(x)$},
every outer y axis line/.append style={black},
every y tick label/.append style={font=\color{black}},
ymin=0.400,
ymax=1.800,
axis background/.style={fill=white}
]
\addplot [ color=black, thick, densely dashed, forget plot]
  table[row sep=crcr]{%
0.001	1.441\\
0.011	1.427\\
0.021	1.413\\
0.031	1.399\\
0.041	1.385\\
0.051	1.372\\
0.061	1.359\\
0.071	1.346\\
0.082	1.334\\
0.092	1.322\\
0.102	1.310\\
0.112	1.298\\
0.122	1.286\\
0.132	1.275\\
0.142	1.263\\
0.152	1.252\\
0.162	1.241\\
0.172	1.231\\
0.182	1.220\\
0.192	1.210\\
0.202	1.200\\
0.212	1.190\\
0.223	1.180\\
0.233	1.170\\
0.243	1.161\\
0.253	1.152\\
0.263	1.142\\
0.273	1.133\\
0.283	1.124\\
0.293	1.116\\
0.303	1.107\\
0.313	1.099\\
0.323	1.090\\
0.333	1.082\\
0.343	1.074\\
0.353	1.066\\
0.364	1.058\\
0.374	1.050\\
0.384	1.043\\
0.394	1.035\\
0.404	1.028\\
0.414	1.020\\
0.424	1.013\\
0.434	1.006\\
0.444	0.999\\
0.454	0.992\\
0.464	0.985\\
0.474	0.979\\
0.484	0.972\\
0.494	0.965\\
0.505	0.959\\
0.515	0.953\\
0.525	0.946\\
0.535	0.940\\
0.545	0.934\\
0.555	0.928\\
0.565	0.922\\
0.575	0.916\\
0.585	0.910\\
0.595	0.904\\
0.605	0.899\\
0.615	0.893\\
0.625	0.888\\
0.635	0.882\\
0.646	0.877\\
0.656	0.871\\
0.666	0.866\\
0.676	0.861\\
0.686	0.856\\
0.696	0.851\\
0.706	0.846\\
0.716	0.841\\
0.726	0.836\\
0.736	0.831\\
0.746	0.826\\
0.756	0.821\\
0.766	0.817\\
0.776	0.812\\
0.787	0.808\\
0.797	0.803\\
0.807	0.799\\
0.817	0.794\\
0.827	0.790\\
0.837	0.785\\
0.847	0.781\\
0.857	0.777\\
0.867	0.773\\
0.877	0.769\\
0.887	0.764\\
0.897	0.760\\
0.907	0.756\\
0.917	0.752\\
0.928	0.748\\
0.938	0.745\\
0.948	0.741\\
0.958	0.737\\
0.968	0.733\\
0.978	0.729\\
0.988	0.726\\
0.998	0.722\\
};
\addplot [  color=red, solid, forget plot]
  table[row sep=crcr]{%
0.001	0.744\\
0.011	0.860\\
0.021	0.968\\
0.031	1.063\\
0.041	1.142\\
0.051	1.202\\
0.061	1.242\\
0.071	1.264\\
0.082	1.270\\
0.092	1.263\\
0.102	1.247\\
0.112	1.225\\
0.122	1.200\\
0.132	1.176\\
0.142	1.154\\
0.152	1.135\\
0.162	1.120\\
0.172	1.110\\
0.182	1.106\\
0.192	1.107\\
0.202	1.113\\
0.212	1.125\\
0.223	1.142\\
0.233	1.163\\
0.243	1.186\\
0.253	1.208\\
0.263	1.228\\
0.273	1.243\\
0.283	1.251\\
0.293	1.251\\
0.303	1.243\\
0.313	1.226\\
0.323	1.203\\
0.333	1.175\\
0.343	1.145\\
0.353	1.114\\
0.364	1.085\\
0.374	1.060\\
0.384	1.038\\
0.394	1.021\\
0.404	1.007\\
0.414	0.996\\
0.424	0.988\\
0.434	0.982\\
0.444	0.977\\
0.454	0.974\\
0.464	0.971\\
0.474	0.968\\
0.484	0.965\\
0.494	0.962\\
0.505	0.958\\
0.515	0.952\\
0.525	0.945\\
0.535	0.937\\
0.545	0.928\\
0.555	0.918\\
0.565	0.908\\
0.575	0.898\\
0.585	0.888\\
0.595	0.880\\
0.605	0.874\\
0.615	0.871\\
0.625	0.870\\
0.635	0.872\\
0.646	0.877\\
0.656	0.883\\
0.666	0.889\\
0.676	0.895\\
0.686	0.899\\
0.696	0.902\\
0.706	0.902\\
0.716	0.900\\
0.726	0.896\\
0.736	0.892\\
0.746	0.885\\
0.756	0.877\\
0.766	0.867\\
0.776	0.855\\
0.787	0.840\\
0.797	0.823\\
0.807	0.803\\
0.817	0.781\\
0.827	0.758\\
0.837	0.736\\
0.847	0.716\\
0.857	0.698\\
0.867	0.684\\
0.877	0.673\\
0.887	0.667\\
0.897	0.663\\
0.907	0.660\\
0.917	0.658\\
0.928	0.653\\
0.938	0.642\\
0.948	0.625\\
0.958	0.599\\
0.968	0.562\\
0.978	0.517\\
0.988	0.463\\
0.998	0.404\\
};
\addplot [dotted, thick, color=blue, forget plot]
  table[row sep=crcr]{%
0.001	0.754\\
0.011	0.907\\
0.021	1.046\\
0.031	1.163\\
0.041	1.252\\
0.051	1.311\\
0.061	1.342\\
0.071	1.348\\
0.082	1.333\\
0.092	1.304\\
0.102	1.267\\
0.112	1.226\\
0.122	1.189\\
0.132	1.157\\
0.142	1.131\\
0.152	1.113\\
0.162	1.099\\
0.172	1.089\\
0.182	1.081\\
0.192	1.078\\
0.202	1.082\\
0.212	1.094\\
0.223	1.116\\
0.233	1.146\\
0.243	1.182\\
0.253	1.220\\
0.263	1.255\\
0.273	1.284\\
0.283	1.301\\
0.293	1.305\\
0.303	1.293\\
0.313	1.267\\
0.323	1.230\\
0.333	1.185\\
0.343	1.139\\
0.353	1.095\\
0.364	1.059\\
0.374	1.032\\
0.384	1.014\\
0.394	1.002\\
0.404	0.994\\
0.414	0.987\\
0.424	0.982\\
0.434	0.976\\
0.444	0.971\\
0.454	0.969\\
0.464	0.968\\
0.474	0.969\\
0.484	0.970\\
0.494	0.969\\
0.505	0.966\\
0.515	0.960\\
0.525	0.951\\
0.535	0.942\\
0.545	0.931\\
0.555	0.921\\
0.565	0.910\\
0.575	0.897\\
0.585	0.884\\
0.595	0.870\\
0.605	0.858\\
0.615	0.851\\
0.625	0.850\\
0.635	0.855\\
0.646	0.866\\
0.656	0.881\\
0.666	0.896\\
0.676	0.908\\
0.686	0.916\\
0.696	0.917\\
0.706	0.913\\
0.716	0.906\\
0.726	0.899\\
0.736	0.893\\
0.746	0.888\\
0.756	0.884\\
0.766	0.880\\
0.776	0.872\\
0.787	0.859\\
0.797	0.840\\
0.807	0.815\\
0.817	0.787\\
0.827	0.756\\
0.837	0.727\\
0.847	0.700\\
0.857	0.677\\
0.867	0.660\\
0.877	0.650\\
0.887	0.646\\
0.897	0.648\\
0.907	0.655\\
0.917	0.665\\
0.928	0.676\\
0.938	0.682\\
0.948	0.678\\
0.958	0.660\\
0.968	0.625\\
0.978	0.570\\
0.988	0.500\\
0.998	0.418\\
};
\end{axis}
\end{tikzpicture}%
\end{minipage}%\	
\\
\hspace{-6cm}\begin{minipage}[b]{0.3\textwidth}
                   \centering
\begin{tikzpicture}

\begin{axis}[%
width=4.2in,
height=2.0in,
at={(0.758in,0.474in)},
scale only axis,
separate axis lines,
every outer x axis line/.append style={black},
every x tick label/.append style={font=\color{black}},
xmin=0.000,
xmax=1.000,
xlabel ={$x$},
ylabel = {$f_{D,h}(x)$},
every outer y axis line/.append style={black},
every y tick label/.append style={font=\color{black}},
ymin=0.200,
ymax=1.600,
axis background/.style={fill=white}
]
\addplot [ color=black, thick, densely dashed, forget plot]
  table[row sep=crcr]{%
0.001	1.441\\
0.011	1.427\\
0.021	1.413\\
0.031	1.399\\
0.041	1.385\\
0.051	1.372\\
0.061	1.359\\
0.071	1.346\\
0.082	1.334\\
0.092	1.322\\
0.102	1.310\\
0.112	1.298\\
0.122	1.286\\
0.132	1.275\\
0.142	1.263\\
0.152	1.252\\
0.162	1.241\\
0.172	1.231\\
0.182	1.220\\
0.192	1.210\\
0.202	1.200\\
0.212	1.190\\
0.223	1.180\\
0.233	1.170\\
0.243	1.161\\
0.253	1.152\\
0.263	1.142\\
0.273	1.133\\
0.283	1.124\\
0.293	1.116\\
0.303	1.107\\
0.313	1.099\\
0.323	1.090\\
0.333	1.082\\
0.343	1.074\\
0.353	1.066\\
0.364	1.058\\
0.374	1.050\\
0.384	1.043\\
0.394	1.035\\
0.404	1.028\\
0.414	1.020\\
0.424	1.013\\
0.434	1.006\\
0.444	0.999\\
0.454	0.992\\
0.464	0.985\\
0.474	0.979\\
0.484	0.972\\
0.494	0.965\\
0.505	0.959\\
0.515	0.953\\
0.525	0.946\\
0.535	0.940\\
0.545	0.934\\
0.555	0.928\\
0.565	0.922\\
0.575	0.916\\
0.585	0.910\\
0.595	0.904\\
0.605	0.899\\
0.615	0.893\\
0.625	0.888\\
0.635	0.882\\
0.646	0.877\\
0.656	0.871\\
0.666	0.866\\
0.676	0.861\\
0.686	0.856\\
0.696	0.851\\
0.706	0.846\\
0.716	0.841\\
0.726	0.836\\
0.736	0.831\\
0.746	0.826\\
0.756	0.821\\
0.766	0.817\\
0.776	0.812\\
0.787	0.808\\
0.797	0.803\\
0.807	0.799\\
0.817	0.794\\
0.827	0.790\\
0.837	0.785\\
0.847	0.781\\
0.857	0.777\\
0.867	0.773\\
0.877	0.769\\
0.887	0.764\\
0.897	0.760\\
0.907	0.756\\
0.917	0.752\\
0.928	0.748\\
0.938	0.745\\
0.948	0.741\\
0.958	0.737\\
0.968	0.733\\
0.978	0.729\\
0.988	0.726\\
0.998	0.722\\
};
\addplot [  color=red, solid, forget plot]
  table[row sep=crcr]{%
0.001	0.678\\
0.011	0.762\\
0.021	0.843\\
0.031	0.920\\
0.041	0.989\\
0.051	1.050\\
0.061	1.103\\
0.071	1.146\\
0.082	1.181\\
0.092	1.207\\
0.102	1.225\\
0.112	1.237\\
0.122	1.243\\
0.132	1.245\\
0.142	1.243\\
0.152	1.238\\
0.162	1.231\\
0.172	1.222\\
0.182	1.213\\
0.192	1.204\\
0.202	1.194\\
0.212	1.186\\
0.223	1.177\\
0.233	1.169\\
0.243	1.162\\
0.253	1.156\\
0.263	1.150\\
0.273	1.144\\
0.283	1.139\\
0.293	1.134\\
0.303	1.128\\
0.313	1.123\\
0.323	1.117\\
0.333	1.110\\
0.343	1.103\\
0.353	1.095\\
0.364	1.087\\
0.374	1.078\\
0.384	1.069\\
0.394	1.059\\
0.404	1.049\\
0.414	1.040\\
0.424	1.030\\
0.434	1.021\\
0.444	1.011\\
0.454	1.003\\
0.464	0.994\\
0.474	0.985\\
0.484	0.977\\
0.494	0.969\\
0.505	0.961\\
0.515	0.953\\
0.525	0.946\\
0.535	0.938\\
0.545	0.931\\
0.555	0.924\\
0.565	0.918\\
0.575	0.912\\
0.585	0.906\\
0.595	0.901\\
0.605	0.897\\
0.615	0.893\\
0.625	0.890\\
0.635	0.887\\
0.646	0.884\\
0.656	0.882\\
0.666	0.879\\
0.676	0.877\\
0.686	0.874\\
0.696	0.871\\
0.706	0.867\\
0.716	0.863\\
0.726	0.859\\
0.736	0.853\\
0.746	0.847\\
0.756	0.841\\
0.766	0.834\\
0.776	0.827\\
0.787	0.820\\
0.797	0.813\\
0.807	0.807\\
0.817	0.800\\
0.827	0.793\\
0.837	0.786\\
0.847	0.779\\
0.857	0.771\\
0.867	0.761\\
0.877	0.750\\
0.887	0.737\\
0.897	0.721\\
0.907	0.701\\
0.917	0.679\\
0.928	0.652\\
0.938	0.621\\
0.948	0.587\\
0.958	0.548\\
0.968	0.506\\
0.978	0.461\\
0.988	0.414\\
0.998	0.366\\
};
\addplot [dotted, thick, color=blue, forget plot]
  table[row sep=crcr]{%
0.001	0.711\\
0.011	0.928\\
0.021	1.105\\
0.031	1.224\\
0.041	1.289\\
0.051	1.312\\
0.061	1.309\\
0.071	1.299\\
0.082	1.291\\
0.092	1.291\\
0.102	1.295\\
0.112	1.302\\
0.122	1.309\\
0.132	1.311\\
0.142	1.305\\
0.152	1.285\\
0.162	1.254\\
0.172	1.217\\
0.182	1.181\\
0.192	1.153\\
0.202	1.142\\
0.212	1.149\\
0.223	1.166\\
0.233	1.179\\
0.243	1.179\\
0.253	1.164\\
0.263	1.142\\
0.273	1.124\\
0.283	1.116\\
0.293	1.120\\
0.303	1.127\\
0.313	1.132\\
0.323	1.134\\
0.333	1.137\\
0.343	1.138\\
0.353	1.133\\
0.364	1.117\\
0.374	1.092\\
0.384	1.066\\
0.394	1.045\\
0.404	1.030\\
0.414	1.020\\
0.424	1.015\\
0.434	1.012\\
0.444	1.009\\
0.454	1.004\\
0.464	0.996\\
0.474	0.985\\
0.484	0.972\\
0.494	0.963\\
0.505	0.961\\
0.515	0.962\\
0.525	0.957\\
0.535	0.942\\
0.545	0.921\\
0.555	0.905\\
0.565	0.901\\
0.575	0.902\\
0.585	0.900\\
0.595	0.891\\
0.605	0.881\\
0.615	0.877\\
0.625	0.881\\
0.635	0.886\\
0.646	0.887\\
0.656	0.881\\
0.666	0.873\\
0.676	0.865\\
0.686	0.864\\
0.696	0.871\\
0.706	0.883\\
0.716	0.893\\
0.726	0.896\\
0.736	0.892\\
0.746	0.880\\
0.756	0.862\\
0.766	0.839\\
0.776	0.813\\
0.787	0.791\\
0.797	0.778\\
0.807	0.774\\
0.817	0.776\\
0.827	0.781\\
0.837	0.790\\
0.847	0.798\\
0.857	0.804\\
0.867	0.804\\
0.877	0.792\\
0.887	0.772\\
0.897	0.748\\
0.907	0.732\\
0.917	0.726\\
0.928	0.726\\
0.938	0.724\\
0.948	0.714\\
0.958	0.691\\
0.968	0.648\\
0.978	0.578\\
0.988	0.481\\
0.998	0.365\\
};
\end{axis}
\end{tikzpicture}%
\end{minipage}%\
\\
\hspace{-6cm}\begin{minipage}[b]{0.3\textwidth}
                          \centering
\begin{tikzpicture}

\begin{axis}[%
width=4.2in,
height=2.0in,
at={(2.578in,1.098in)},
scale only axis,
separate axis lines,
every outer x axis line/.append style={black},
every x tick label/.append style={font=\color{black}},
xmin=0.000,
xmax=1.000,
xlabel ={$x$},
ylabel = {$f_{D,h}(x)$},
every outer y axis line/.append style={black},
every y tick label/.append style={font=\color{black}},
ymin=0.200,
ymax=1.600,
axis background/.style={fill=white}
]
\addplot [ color=black, thick, densely dashed, forget plot]
  table[row sep=crcr]{%
0.001	1.441\\
0.011	1.427\\
0.021	1.413\\
0.031	1.399\\
0.041	1.385\\
0.051	1.372\\
0.061	1.359\\
0.071	1.346\\
0.082	1.334\\
0.092	1.322\\
0.102	1.310\\
0.112	1.298\\
0.122	1.286\\
0.132	1.275\\
0.142	1.263\\
0.152	1.252\\
0.162	1.241\\
0.172	1.231\\
0.182	1.220\\
0.192	1.210\\
0.202	1.200\\
0.212	1.190\\
0.223	1.180\\
0.233	1.170\\
0.243	1.161\\
0.253	1.152\\
0.263	1.142\\
0.273	1.133\\
0.283	1.124\\
0.293	1.116\\
0.303	1.107\\
0.313	1.099\\
0.323	1.090\\
0.333	1.082\\
0.343	1.074\\
0.353	1.066\\
0.364	1.058\\
0.374	1.050\\
0.384	1.043\\
0.394	1.035\\
0.404	1.028\\
0.414	1.020\\
0.424	1.013\\
0.434	1.006\\
0.444	0.999\\
0.454	0.992\\
0.464	0.985\\
0.474	0.979\\
0.484	0.972\\
0.494	0.965\\
0.505	0.959\\
0.515	0.953\\
0.525	0.946\\
0.535	0.940\\
0.545	0.934\\
0.555	0.928\\
0.565	0.922\\
0.575	0.916\\
0.585	0.910\\
0.595	0.904\\
0.605	0.899\\
0.615	0.893\\
0.625	0.888\\
0.635	0.882\\
0.646	0.877\\
0.656	0.871\\
0.666	0.866\\
0.676	0.861\\
0.686	0.856\\
0.696	0.851\\
0.706	0.846\\
0.716	0.841\\
0.726	0.836\\
0.736	0.831\\
0.746	0.826\\
0.756	0.821\\
0.766	0.817\\
0.776	0.812\\
0.787	0.808\\
0.797	0.803\\
0.807	0.799\\
0.817	0.794\\
0.827	0.790\\
0.837	0.785\\
0.847	0.781\\
0.857	0.777\\
0.867	0.773\\
0.877	0.769\\
0.887	0.764\\
0.897	0.760\\
0.907	0.756\\
0.917	0.752\\
0.928	0.748\\
0.938	0.745\\
0.948	0.741\\
0.958	0.737\\
0.968	0.733\\
0.978	0.729\\
0.988	0.726\\
0.998	0.722\\
};
\addplot [  color=red, solid, forget plot]
  table[row sep=crcr]{%
0.001	0.719\\
0.011	0.983\\
0.021	1.187\\
0.031	1.304\\
0.041	1.347\\
0.051	1.350\\
0.061	1.341\\
0.071	1.334\\
0.082	1.328\\
0.092	1.317\\
0.102	1.303\\
0.112	1.287\\
0.122	1.275\\
0.132	1.267\\
0.142	1.264\\
0.152	1.262\\
0.162	1.258\\
0.172	1.250\\
0.182	1.240\\
0.192	1.231\\
0.202	1.220\\
0.212	1.206\\
0.223	1.187\\
0.233	1.169\\
0.243	1.154\\
0.253	1.145\\
0.263	1.139\\
0.273	1.131\\
0.283	1.120\\
0.293	1.111\\
0.303	1.110\\
0.313	1.115\\
0.323	1.120\\
0.333	1.116\\
0.343	1.101\\
0.353	1.080\\
0.364	1.059\\
0.374	1.040\\
0.384	1.027\\
0.394	1.020\\
0.404	1.019\\
0.414	1.019\\
0.424	1.015\\
0.434	1.006\\
0.444	0.992\\
0.454	0.977\\
0.464	0.963\\
0.474	0.954\\
0.484	0.950\\
0.494	0.949\\
0.505	0.948\\
0.515	0.945\\
0.525	0.942\\
0.535	0.941\\
0.545	0.942\\
0.555	0.942\\
0.565	0.936\\
0.575	0.925\\
0.585	0.911\\
0.595	0.899\\
0.605	0.892\\
0.615	0.890\\
0.625	0.891\\
0.635	0.890\\
0.646	0.885\\
0.656	0.878\\
0.666	0.873\\
0.676	0.867\\
0.686	0.859\\
0.696	0.851\\
0.706	0.845\\
0.716	0.843\\
0.726	0.843\\
0.736	0.843\\
0.746	0.840\\
0.756	0.834\\
0.766	0.827\\
0.776	0.820\\
0.787	0.812\\
0.797	0.803\\
0.807	0.795\\
0.817	0.790\\
0.827	0.789\\
0.837	0.789\\
0.847	0.790\\
0.857	0.791\\
0.867	0.791\\
0.877	0.788\\
0.887	0.781\\
0.897	0.772\\
0.907	0.763\\
0.917	0.753\\
0.928	0.743\\
0.938	0.734\\
0.948	0.723\\
0.958	0.707\\
0.968	0.676\\
0.978	0.616\\
0.988	0.517\\
0.998	0.386\\
};
\addplot [dotted, thick, color=blue, forget plot]
  table[row sep=crcr]{%
0.001	0.718\\
0.011	1.073\\
0.021	1.309\\
0.031	1.388\\
0.041	1.379\\
0.051	1.350\\
0.061	1.328\\
0.071	1.328\\
0.082	1.340\\
0.092	1.328\\
0.102	1.299\\
0.112	1.283\\
0.122	1.270\\
0.132	1.260\\
0.142	1.261\\
0.152	1.266\\
0.162	1.266\\
0.172	1.252\\
0.182	1.235\\
0.192	1.232\\
0.202	1.227\\
0.212	1.212\\
0.223	1.191\\
0.233	1.162\\
0.243	1.142\\
0.253	1.142\\
0.263	1.144\\
0.273	1.138\\
0.283	1.119\\
0.293	1.099\\
0.303	1.097\\
0.313	1.114\\
0.323	1.137\\
0.333	1.135\\
0.343	1.103\\
0.353	1.076\\
0.364	1.060\\
0.374	1.036\\
0.384	1.015\\
0.394	1.012\\
0.404	1.020\\
0.414	1.023\\
0.424	1.022\\
0.434	1.013\\
0.444	0.993\\
0.454	0.974\\
0.464	0.959\\
0.474	0.944\\
0.484	0.946\\
0.494	0.953\\
0.505	0.951\\
0.515	0.945\\
0.525	0.937\\
0.535	0.935\\
0.545	0.945\\
0.555	0.951\\
0.565	0.945\\
0.575	0.927\\
0.585	0.907\\
0.595	0.894\\
0.605	0.884\\
0.615	0.882\\
0.625	0.896\\
0.635	0.901\\
0.646	0.883\\
0.656	0.872\\
0.666	0.873\\
0.676	0.872\\
0.686	0.864\\
0.696	0.846\\
0.706	0.836\\
0.716	0.841\\
0.726	0.847\\
0.736	0.845\\
0.746	0.842\\
0.756	0.838\\
0.766	0.826\\
0.776	0.819\\
0.787	0.816\\
0.797	0.804\\
0.807	0.789\\
0.817	0.783\\
0.827	0.787\\
0.837	0.793\\
0.847	0.789\\
0.857	0.787\\
0.867	0.798\\
0.877	0.798\\
0.887	0.781\\
0.897	0.770\\
0.907	0.765\\
0.917	0.752\\
0.928	0.742\\
0.938	0.736\\
0.948	0.724\\
0.958	0.714\\
0.968	0.704\\
0.978	0.669\\
0.988	0.570\\
0.998	0.396\\
};
\end{axis}
\end{tikzpicture}%
\end{minipage}%\
\caption{Plots of the kernel density estimators $f_{D,h}$ for Gauss map  in Example \ref{gauss_map} with different bandwidths and its true density with different sample sizes. The sample size of each panel, from up to bottom, is $10^3$, $10^4$ and $10^5$, respectively. In each panel, the dashed black curve represents the true density of Gauss map, the dotted blue curve is the estimated density of Gauss map with the bandwidth selected by the baseline method while the solid red curve stands for the estimated density of Gauss map with the bandwidth selected by the double kernel method. All curves are plotted with $100$ equispaced points in the interval $(0,1)$.}\label{gauss_map_figure}  
\end{figure}  

From Tables \ref{logistic_map_table} and \ref{gauss_map_table}, and Figs.\,\ref{logistic_map_figure} and \ref{gauss_map_figure}, we see that the true density functions of Logistic map and Gauss map can be approximated well with enough observations and the double kernel method works slightly better than the other three methods for the two dynamical systems. In fact, according to our experimental experience, we find that the bandwidth selector of the kernel density estimator for a dynamical system is usually ad-hoc. That is, for existing bandwidth selectors, there seems no a universal optimal one that can be applicable to all dynamical systems and outperforms the others. Therefore, further exploration and insights on the bandwidth selection problem in the dynamical system context certainly deserve future study. On the other hand, we also notice that due to the presence of dependence among observations generated by dynamical systems, the sample size usually needs to be large enough to approximate the density function well. This can be also seen from the plotted density functions in Figs.\,\ref{logistic_map_figure} and \ref{gauss_map_figure} with varying sample sizes. 

Aside from the above observations, not surprisingly, from Figs.\,\ref{logistic_map_figure} and \ref{gauss_map_figure}, we also observe the \textit{boundary effect} \citep{gasser1985kernels} from the kernel density estimators for dynamical systems, which seems to be even more significant than the i.i.d case. From a practical implementation view, some special studies are arguably called for addressing this problem.

\section{Proofs of Section \ref{sec::consistency_convergence}}\label{proofs}

\begin{proof}[of Theorem \ref{ApproximationError}]
\textit{(i)}
Since the space of continuous and compactly supported functions  $C_c(\mathbb{R}^d)$  is dense in $L_1(\mathbb{R}^d)$, we can find $\bar{f} \in C_c(\mathbb{R}^d)$ such that 
\begin{align*}
\|f-\bar{f}\|_1 \leq \varepsilon/3, \,\,\,	\forall \varepsilon >0.
\end{align*}
Therefore, for any $\varepsilon>0$, we have
\begin{align}\label{I_1}
\begin{split}
\|f_{P,h} - f\|_1
& = \int_{\mathbb{R}^d} |f * K_h - f| \, \mathrm{d}x
\\
& \leq \int_{\mathbb{R}^d} |f * K_h - \bar{f} * K_h| \, \mathrm{d}x
       + \int_{\mathbb{R}^d} |\bar{f} * K_h - \bar{f}| \, \mathrm{d}x
       + \int_{\mathbb{R}^d} |f-\bar{f}| \, \mathrm{d}x
\\
& \leq \frac{2 \varepsilon}{3} + \int_{\mathbb{R}^d} |\bar{f} * K_h - \bar{f}| \, \mathrm{d}x,
\end{split}
\end{align}
where $K_h$ is defined in \eqref{K_h} and
the last inequality follows from the fact that
\begin{align*}
\|f * K_h - \bar{f} * K_h\|_1\leq \|f-\bar{f}\|_1\leq \varepsilon/3.
\end{align*}
The above inequality is due to 
Young's inequality (8.7) in \cite{Folland99}. 
Moreover, there exist a  constant $M>0$ such that $\mathrm{supp}(\bar{f}) \subset B_M$ and a constant $r > 0$ such that
\begin{align*}
\int_{H_r} K(\|x\|) \, \mathrm{d}x
\leq \frac{\varepsilon}{9 \|\bar{f}\|_1}.
\end{align*}

Now we define $L : \mathbb{R}^d \to [0, \infty)$ by
\begin{align*}
L(x) := \eins_{[-r,r]}(\|x\|) K(\|x\|)
\end{align*}
and $L_h : \mathbb{R}^d \to [0, \infty)$ by
\begin{align*}
L_h(x) := h^{-d} L (x/h).
\end{align*}
Then we have
\begin{align*} 
\begin{split}
\int_{\mathbb{R}^d} |\bar{f} * K_h - \bar{f}| \, \mathrm{d}x
& \leq \int_{\mathbb{R}^d} |\bar{f} * K_h - \bar{f} * L_h| \, \mathrm{d}x
       + \int_{\mathbb{R}^d} | \bar{f} * L_h - \bar{f}| \, \mathrm{d}x
\\
& \leq \|\bar{f}\|_1 \|K_h - L_h\|_1
       + \int_{\mathbb{R}^d} \biggl| \bar{f} * L_h - \bar{f} \int_{\mathbb{R}^d} L_h \, \mathrm{d}x \biggr| \, \mathrm{d}x
\\
& \phantom{=}        
       + \int_{\mathbb{R}^d} \biggl| \bar{f} * \int_{\mathbb{R}^d} (L_h - K_h) \, \mathrm{d}x \biggr| \, \mathrm{d}x
\\
& \leq 2 \|\bar{f}\|_1 \|K_h - L_h\|_1
       + \int_{\mathbb{R}^d} \biggl| \bar{f} * L_h	 - \bar{f} \int_{\mathbb{R}^d} L_h \, \mathrm{d}x \biggr| \, \mathrm{d}x.
\end{split}
\end{align*}
Moreover, we have
\begin{align*} 
\begin{split}
\|K_h - L_h\|_1
& = \int_{\mathbb{R}^d} \frac{1}{h^d} \biggl| \eins_{[-r,r]} \bigg( \frac{\|x\|}{h} \bigg) K \bigg( \frac{\|x\|}{h} \bigg)
    - K \bigg( \frac{\|x\|}{h} \bigg) \biggr| \, \mathrm{d}x
\\
& = \int_{\mathbb{R}^d} \bigl| \eins_{[-r,r]}(\|x\|) K(\|x\|) - K(\|x\|) \bigr| \, \mathrm{d}x
\\
& = \int_{H_r} K(\|x\|) \, \mathrm{d}x
\leq \frac{\varepsilon}{9 \|\bar{f}\|_1}.
\end{split}
\end{align*}
Finally, for $h \leq 1$, we have
\begin{align*}
\int_{\mathbb{R}^d} \biggl| \bar{f} * L_h - \bar{f} \int_{\mathbb{R}^d} L_h \, \mathrm{d}x \biggr| \, \mathrm{d}x
& = \int_{\mathbb{R}^d} \biggl| \int_{\mathbb{R}^d}
    \bigl( \bar{f}(x-x') - \bar{f}(x) \bigr) L_h(x') \, \mathrm{d}x' \biggr| \, \mathrm{d}x
\\
& \leq \int_{B_{r+M}} \int_{\mathbb{R}^d} |\bar{f}(x-x') - \bar{f}(x)| L_h(x') \, \mathrm{d}x' \, \mathrm{d}x.
\end{align*}
Since $\bar{f}$ is uniformly continuous, there exists a constant $h_\varepsilon > 0$
such that for all $h \leq h_\varepsilon$ and $\|x'\| \leq r h$, we have
\begin{align*}
|\bar{f}(x-x') - \bar{f}(x)| \leq \varepsilon' := \frac{\varepsilon}{9 (r+M)^d \lambda^d(B_1)}.
\end{align*}
Consequently we obtain
\begin{align*}
\int_{\mathbb{R}^d} |\bar{f}(x-x') - \bar{f}(x)| L_h(x') \, \mathrm{d}x'
 \leq \varepsilon' \int_{B_{r h}} L_h(x') \, \mathrm{d}x'
 \leq \varepsilon' \int_{\mathbb{R}^d} K_h \, \mathrm{d}x
 = \varepsilon'.
\end{align*}
Therefore, we obtain
\begin{align}\label{I_IV}
\int_{\mathbb{R}^d} \biggl| \bar{f} * L_h - \bar{f} \int_{\mathbb{R}^d} L_h \, \mathrm{d}x \biggr| \, \mathrm{d}x 
\leq \int_{B_{r+M}} \varepsilon' \, \mathrm{d}x = \frac{\varepsilon}{9}
\end{align}
and consequently the assertion can be proved by combining estimates in \eqref{I_1} and \eqref{I_IV}.

\textit{(ii)}
The $\alpha$-H\"{o}lder continuity of $f$ tells us that for any $x\in\mathbb{R}^d$, there holds
\begin{align*}
|f_{P,h}(x) - f(x)|
& = \biggl| \frac{1}{h^d} \int_{\mathbb{R}^d} K \biggl( \frac{\|x-x'\|}{h} \biggr) f(x') \, \mathrm{d}x' - f(x) \biggr|
\\
& = \biggl| \int_{\mathbb{R}^d} K(\|x'\|) f(x + h x') \, \mathrm{d}x' - f(x) \biggr|
\\
& = \biggl| \int_{\mathbb{R}^d} K(\|x'\|) \bigl( f(x + h x') - f(x) \bigr) \, \mathrm{d}x' \biggr|
\\
&  \lesssim  	\int_{\mathbb{R}^d} K(\|x'\|)   \bigl( h \|x'\| \bigr)^{\alpha} \, \mathrm{d}x'
\\
& \lesssim \int_{\mathbb{R}^d} K \bigl( \|x'\|_{\ell_2^d} \bigr) h^{\alpha} \|x'\|_{\ell_2^d}^{\alpha} \, \mathrm{d}x'
\\
& \lesssim h^{\alpha} \int_0^\infty K(r) r^{\alpha+d-1} \, \mathrm{d}r
\lesssim h^{\alpha}.
\end{align*}
We thus have completed the proof of Theorem \ref{ApproximationError}.
\end{proof}

The following lemma, which will be used several times in the sequel,
supplies the key to the proof of Theorem \ref{ComparisionNorms}.

\begin{lemma}\label{GeneralRandEstimate}
Let the assumptions of Theorem \ref{ComparisionNorms} hold and $k_{x,h}$ be defined in \eqref{k_x_h}.
Then, for an arbitrary probability measure $Q$ on $\mathbb{R}^d$, we have
 \begin{align*}
  \int_{H_r} \mathbb{E}_Q k_{x,h} \, \mathrm{d}x
  \lesssim   Q (H_{r/2})
       +   \left( h/r\right)^{\beta}.
 \end{align*}
\end{lemma}

\begin{proof}[of Lemma \ref{GeneralRandEstimate}]
For a positive constant $t_0$, we have
\begin{align*}
\int_{H_r} \mathbb{E}_P k_{x,h} \, \mathrm{d}x
&  = \int_{H_r} \int_{\mathbb{R}^d}
    h^{-d} K \left(  \|x - x'\|/h \right) \, \mathrm{d}P(x') \, \mathrm{d}x
\\
&  = \int_{\mathbb{R}^d} \int_{\mathbb{R}^d}   K(\|x\|)
    \eins_{H_r}(h x + x')
    \, \mathrm{d}x \, \mathrm{d}P(x')
\\
&  = \int_{\mathbb{R}^d} K(\|x\|) \int_{\mathbb{R}^d}
    \eins_{H_r}(h x + x')
    \,  \mathrm{d}P(x') \, \mathrm{d}x
\\
&  \leq \int_{B_{t_0}}  K(\|x\|)	 \int_{\mathbb{R}^d}
    \eins_{H_r}(h x + x') \, \mathrm{d}P(x')\mathrm{d}x
    + \int_{H_{t_0}} K(\|x\|) \,  \mathrm{d}x.
\end{align*}
On the other hand, it is easy to see that $\eins_{H_r}(h x + x') = 1$ if and only if $\|h x + x'\| \geq r$. 
Now we set $t_0 := \frac{r}{2h}$. In this case, if we additionally have $x \in B_{t_0}$, 
then $\|x'\| \geq r - h \|x\| \geq r - h t_0 = r/2$. Therefore, we come to the following estimate 
\begin{align}\label{RandEstimate}
\begin{split}
\int_{H_r} \mathbb{E}_P k_{x,h} \, \mathrm{d}x
&  \leq \int_{B_{t_0}} K(\|x\|) P (H_{r/2}) \, \mathrm{d}x
+ \int_{H_{t_0}} K(\|x\|) \, \mathrm{d}x
\\
&  \lesssim   P (H_{r/2})
       +  \int_{t_0}^{\infty} K(t) t^{d-1} \, \mathrm{d}t
\\
& \lesssim   P (H_{r/2})
       +   \int_{t_0}^{\infty} K(t) t_0^{-\beta} t^{d+\beta-1} \, \mathrm{d}t
\\
&  \lesssim   P (H_{r/2})
       +   \, t_0^{-\beta}
\\
&  \lesssim   P (H_{r/2})
       +   ( h/r )^{\beta}.
 \end{split}
\end{align}
We thus have shown the assertion of Lemma \ref{GeneralRandEstimate}.
\end{proof}

\begin{proof}[of Theorem \ref{ComparisionNorms}]
We decompose $\|f_{P,h} - f\|_1$ as follows
\begin{align}\label{temp_I}
\begin{split}
\|f_{P,h} - f\|_1
& = \int_{B_r} |f_{P,h} - f| \, \mathrm{d}x
  + \int_{H_r} |f_{P,h} - f| \, \mathrm{d}x
\\
& \leq \lambda^d (B_r) \|f_{P,h} - f\|_{\infty}
       + \int_{H_r} \mathbb{E}_P k_{x,h} \, \mathrm{d}x
       + \int_{H_r} f \, \mathrm{d}x
\\
& \leq r^d \|f_{P,h} - f\|_{\infty}
       + \int_{H_r} \mathbb{E}_P k_{x,h} \, \mathrm{d}x
       + P \left( H_r \right).
\end{split}
\end{align}
Combining the two estimates in \eqref{temp_I} and \eqref{RandEstimate}, we obtain the desired conclusion.
\end{proof}

To prove Proposition \ref{kernelSpaceCN}, we need the following lemmas.

\begin{lemma}  \label{transNumber}
Let $(X, d)$ and $(Y, e)$ be metric spaces and $T : X \to Y$ be an $\alpha$-H\"{o}lder
continuous function with constant $c$. Then, for $A \subset X$
and all $\varepsilon>0$ we have
\begin{align*}
\mathcal{N}(T(A),e,c\varepsilon^\alpha)
\leq \mathcal{N}(A,d,\varepsilon).
\end{align*}
\end{lemma}

\begin{proof}[of Lemma \ref{transNumber}]
Let $x_1, \ldots, x_n$ be an $\varepsilon$-net of $A$, that is,
$A \subset \bigcup_{i=1}^n B_d(x_i,\varepsilon)$. For $i=1,\cdots,n$, we set $y_i := T(x_i)$. Now, it only suffices to show that this gives a $c\varepsilon^\alpha$-net of $T(A)$.

In fact, supposing that $y \in T(B_d(x_i,\varepsilon))$, then there exists  
$x \in B_d(x_i,\varepsilon)$ such that $T(x) = y$. This implies
\begin{align*}
	e(T(x), T(x_i)) \leq c d^\alpha(x, x_i) \leq c \varepsilon^\alpha.
\end{align*}
Therefore, we have $T(B_d(x_i,\varepsilon)) \subset B_e(y_i,c\varepsilon^\alpha)$. That is, $y_1,\ldots,y_n$ is a $c\varepsilon^\alpha$-net of $T(A)$. This completes the proof of Lemma \ref{transNumber}.
\end{proof}

\begin{remark}\label{remark_covering number}
We remark that when $X$ is a Banach space with the norm $\|\cdot\|$, then for any $c > 0$ there holds
\begin{align*}
\mathcal{N}(c A, \|\cdot\|, \varepsilon)
= \mathcal{N} \left( A, \|\cdot\|,  \varepsilon/c \right).
\end{align*}
\end{remark}

\begin{lemma}\label{DifferentNormCN}
Let $\|\cdot\|'$ be another norm on $\mathbb{R}^d$.
Then  for all
$\varepsilon \in (0,1]$ we have
\begin{align*}
\mathcal{N}(B_1, \|\cdot\|', \varepsilon)
\lesssim   \varepsilon^{-d}.
\end{align*}
\end{lemma}

\begin{proof}[of Lemma \ref{DifferentNormCN}]
It is a straightforward conclusion of Proposition 1.3.1 in \cite{carl1990entropy}
and Lemma 6.21 in \cite{steinwart2008support}.
\end{proof}

\begin{lemma}\label{HolderConstant}
Let $K$ be a $d$-dimensional smoothing kernel that satisfies Conditions $(i)$ in Assumption \ref{assump_kernel}. Let $h > 0$ be the bandwidth parameter, and $k_{x,h}$ be defined in \eqref{k_x_h} for any $x\in\mathbb{R}^d$. Then we have
\begin{align*}
\sup_{y \in \mathbb{R}^d} |k_{x,h}(y) - k_{x',h}(y)|
\leq \frac{c}{h^{\beta+d}} \|x - x'\|^{\beta}, \,\,\,x, x^\prime \in\mathbb{R}^d,
\end{align*}
where $c$ is a positive constant.
\end{lemma}
\begin{proof}[of Lemma \ref{HolderConstant}]
From the definition of $k_{x,h}$ and the fact that $K$ is a $d$-dimensional $\beta$-H\"{o}lder continuous kernel, we have
\begin{align*}
|k_{x,h}(y) - k_{x',h}(y)|
& = \frac{1}{h^d} \biggl| K \bigg( \frac{\|x - y\|}{h} \bigg) - K \bigg( \frac{\|x' - y\|}{h} \bigg) \biggr|
\\
& \leq \frac{c}{h^d} \biggl| \frac{\|x - y\|}{h} - \frac{\|x' - y\|}{h} \biggr|^{\beta}
\\
& \leq c h^{-(\beta+d)} \|x - x'\|^{\beta}, 
\end{align*}
where $c$ is a positive constant. The desired conclusion is thus obtained.
\end{proof}

\begin{proof}[of Proposition \ref{kernelSpaceCN}]
Lemma \ref{HolderConstant} reveals that $\mathcal{K}_{h,r}$ is the image of H\"{o}lder continuous map $B_r \to L_\infty(\mathbb{R}^d)$ with the constant $c h^{-(\beta+d)}$. By Lemmas \ref{transNumber} and \ref{DifferentNormCN} we obtain
\begin{align*}
\mathcal{N} (\mathcal{K}_{h,r}, \|\cdot\|_\infty, \varepsilon)
& \leq \mathcal{N} \left( B_r, \|\cdot\|, \biggl( \frac{\varepsilon h^{\beta+d}}{c}  \biggr)^{1/\beta} \right)
\\
& = \mathcal{N} \biggl( B_1, \|\cdot\|, \biggl( \frac{\varepsilon h^{\beta+d}}{c r^{\beta}}  \biggr)^{1/\beta} \biggr)
\\
& \leq c' \left( \frac{\varepsilon h^{\beta+d}}{r^{\beta}}  \right)^{-d/\beta }, 
\end{align*}
where $c^\prime$ is a   constant independent of $\varepsilon$. This completes the proof of Proposition \ref{kernelSpaceCN}.
\end{proof}

The following Bernstein-type exponential inequality, which was developed recently in \cite{hang2015bernstein}, will serve as one of the main ingredients in the consistency and convergence analysis of the kernel density estimator \eqref{KDE_formal}. It can be stated in the following general form:

\begin{theorem}[\bf Bernstein  Inequality \citep{hang2015bernstein}]\label{bernsteininequality}
Assume that $\mathcal{X}$  $:= (X_n)_{n \geq 1}$ is an $X$-valued stationary geometrically 
(time-reversed) $\mathcal{C}$-mixing process on $(\Omega, \mathcal A, \mu)$ 
with $\|\cdot\|_{\mathcal{C}}$ be defined by \eqref{lambdanorm} for some semi-norm $\vertiii{\cdot}$
satisfying Condition $(ii)$ in Assumption \ref{seminorm_assump}, and $P := \mu_{X_1}$. 
Moreover, let $g:X\to\mathbb{R}$  be a function such that $g \in \mathcal{C}(X)$ with $\mathbb E_P g = 0$ 
and assume that there exist some $A > 0$, $B > 0$, and $\sigma \geq 0$ such that
$\vertiii{g} \leq A$, $\|g\|_{\infty} \leq B$, and $\mathbb E_P g^2 \leq \sigma^2$. Then, for all $\tau > 0$,
$k \in \mathbb{N}$, and  
\begin{align*}
n \geq n_0 := \max \left\{ \min \biggl\{ m \geq 3 : m \geq \biggl( \frac{808 c_0 (3A \!+ \! B)}{B} \biggr)^{\frac{1}{k}} 
\text{ and } \, \frac{m}{(\log m)^{\frac{2}{\gamma}}} \geq 4 \biggr\}, e^{\frac{k+1}{b}} \right\},
\end{align*}
with probability $\mu$ at least $1-4 e^{- \tau}$, there holds
\begin{align*}
\left|\frac{1}{n} \sum_{i=1}^n g(X_i)\right| 
\leq \sqrt{\frac{8 (\log n)^{\frac{2}{\gamma}} \sigma^2 \tau}{n}} 
     + \frac{8 (\log n)^{\frac{2}{\gamma}} B \tau}{3 n}.  
\end{align*}
\end{theorem}

\begin{proof}[of Theorem \ref{OracleInequalityL1}]
Let the notations  $k_{x,h}$ and $\widetilde{k}_{x,h}$ be defined in \eqref{k_x_h}  and \eqref{k_x_h_mean}, respectively,
that is, $k_{x,h}:= h^{-d} K  (  \|x - \cdot\|/h )$, and $\widetilde{k}_{x,h}  := k_{x,h} - \mathbb{E}_P k_{x,h}$.
We first assume that $x \in \mathbb{R}^d$ is fixed and then estimate $\mathbb{E}_D f_{x,h}$ 
by using Bernstein's inequality in Theorem \ref{bernsteininequality}. 
For this purpose, we shall verify the following conditions:
Obviously, we have $\mathbb{E}_P \widetilde{k}_{x,h} = 0$. Moreover, simple estimates yield
\begin{align*}
 \|\widetilde{k}_{x,h}\|_{\infty} 
       \leq 2 \|k_{x,h}\|_{\infty} 
       \leq 2 h^{-d} \|K\|_{\infty} 
       \leq 2 h^{-d} K(0)
\end{align*}
and 
\begin{align*}
 \mathbb{E}_P \widetilde{k}_{x,h}^2
       \leq \mathbb{E}_P k_{x,h}^2
       = \int_{\mathbb{R}^d} k_{x,h}^2(x^\prime) \mathrm{d}P(x').
\end{align*}
Finally, the first condition in Assumption \ref{seminorm_assump} and Condition $(iii)$ in Assumption \ref{assump_kernel}
imply
\begin{align*}
 \vertiii{\widetilde{k}_{x,h}} 
       \leq \vertiii{k_{x,h}} 
       \leq h^{-d} \sup_{x \in \mathbb{R}^d} \vertiii{K \left(\|x - \cdot\|/h \right)}
       \leq h^{-d}\varphi(h).
\end{align*}
Now we can apply the Bernstein-type inequality in Theorem \ref{bernsteininequality} and obtain that
for $n \geq n_1$, for any fixed $x \in \mathbb{R}^d$, with probability $\mu$ at most $4 e^{-\tau}$, there holds
\begin{align} \label{FirstEstimate}
|\mathbb{E}_D \widetilde{k}_{x,h}|
\geq \sqrt{\frac{8\tau(\log n)^{2/\gamma}  \int_{\mathbb{R}^d} k_{x,h}^2(x^\prime) \mathrm{d}P(x')}{n}} + \frac{16\tau(\log n)^{2/\gamma}  K(0) }{3 h^d n},
\end{align}
where 
\begin{align}\label{nzeroOI}
n_1 := \max \left\{ \min \biggl\{ m \geq 3 : 
m \geq \biggl( \frac{808 c_0 (3  h^{-d}\varphi(h) +   K(0))}{2K(0)} \biggr)^{\frac{1}{d+1}}
\text{ and } 
\frac{m}{(\log m)^{\frac{2}{\gamma}}} \geq 4 \biggr\},
e^{\frac{d+1}{b}} \right\}.
\end{align}
Consider the function set $\widetilde{\mathcal{K}}_{h,r} := \{ \widetilde{k}_{x,h} : x \in B_r \}$. We choose $y_1, \ldots, y_m \in B_r$ such that $\{k_{y_1,h}, \ldots, k_{y_m,h}\}$ is a minimal $\varepsilon/2$-net of $\mathcal{K}_{h,r}=\{k_{x,h} : x \in B_r\}$ with respect to $\|\cdot\|_{\infty}$. Noticing the following relation
\begin{align*}
 \|\widetilde{k}_{x,h} - \widetilde{k}_{y_j,h}\|_{\infty} 
 \leq 2 \|k_{x,h} - k_{y_j,h}\|_{\infty} \leq \varepsilon,
\end{align*}
we know that $\{\widetilde{k}_{y_1,h}, \ldots, \widetilde{k}_{y_m,h}\}$ is an $\varepsilon$-net of $\widetilde{\mathcal{K}}_{h,r}$ with respect to $\|\cdot\|_{\infty}$. Note that here we have $m = \mathcal{N} (\mathcal{K}_{h,r}, \|\cdot\|_\infty, \frac{\varepsilon}{2})$,
since the net is minimal. From Proposition \ref{kernelSpaceCN}, we know that there exists a positive constant $c$ independent of $\varepsilon$ such that $\log m \leq c  \log \frac{r}{h \varepsilon}$. 
From the estimate in  \eqref{FirstEstimate} and a union bound argument, with probability $\mu$ at least $1-4m e^{-\tau}$, the following estimate holds
\begin{align*}  
\sup_{j=1,\ldots,m} |\mathbb{E}_D \widetilde{k}_{y_j,h}|
\leq \sqrt{\frac{8\tau(\log n)^{2/\gamma}  \int_{\mathbb{R}^d} k_{y_j,h}^2(x^\prime) \mathrm{d}P(x') }{n}} 
  + \frac{16\tau(\log n)^{2/\gamma}  K(0)   }{h^d n}.
\end{align*}
By a simple variable transformation, we see that with probability $\mu$ at least $1 - e^{-\tau}$, there holds
\begin{align*}  
\begin{split}
\sup_{j=1,\ldots,m} |\mathbb{E}_D \widetilde{k}_{y_j,h}|
\leq &\sqrt{\frac{8 (\log n)^{2/\gamma}  \int_{\mathbb{R}^d} k_{y_j,h}^2(x^\prime) \mathrm{d}P(x') (\tau + \log (4 m))}{n}} \\
 & + \frac{16(\log n)^{2/\gamma}  K(0) (\tau + \log (4 m))}{h^d n}.
 \end{split}
\end{align*}
Recalling that $\{k_{y_1,h}, \ldots, k_{y_m,h}\}$ is an $\varepsilon/2$-net of $\mathcal{K}_{h,r}$,  this implies that, for any $x \in B_r$, there exists $y_j$ such that $\|k_{x,h} - k_{y_j,h}\|_{\infty} \leq \varepsilon/2$. Then we have 
\begin{align*}
\left| |\mathbb{E}_D \widetilde{k}_{x,h}| - |\mathbb{E}_D \widetilde{k}_{y_j,h}| \right|
& \leq \left|\mathbb{E}_D \widetilde{k}_{x,h} - \mathbb{E}_D \widetilde{k}_{y_j,h}\right|
\\
& \leq |\mathbb{E}_D k_{x,h} - \mathbb{E}_D k_{y_j,h}|
       + |\mathbb{E}_P k_{x,h} - \mathbb{E}_P k_{y_j,h}|
\\
& \leq \|k_{x,h} - k_{y_j,h}\|_{L_1(D)} + \|k_{x,h} - k_{y_j,h}\|_{L_1(P)}
\\
& \leq \varepsilon,
\end{align*}
and consequently
\begin{align}  \label{ThirdEstimate}
|\mathbb{E}_D \widetilde{k}_{x,h}| \leq |\mathbb{E}_D \widetilde{k}_{y_j,h}| + \varepsilon.
\end{align}
By setting $a := 8(\log n)^{2/\gamma}  (\tau + \log (4 m))/n$, we have
\begin{align*}
\left| \sqrt{a \int_{\mathbb{R}^d} k_{x,h}^2(x') \, \mathrm{d}P(x')}
  - \sqrt{a \int_{\mathbb{R}^d} k_{y_j,h}^2(x') \, \mathrm{d}P(x')} \right|
& = \bigl| \|\sqrt{a} k_{x,h}\|_{L_2(P)} - \|\sqrt{a} k_{y_j,h}\|_{L_2(P)} \bigr|
\\
& \leq \sqrt{a} \|k_{x,h} - k_{y_j,h}\|_{L_2(P)}
\\
& \leq \sqrt{a} \varepsilon/2.
\end{align*}
This together with inequality \eqref{ThirdEstimate} implies that for any $x \in B_r$, there holds
\begin{align*}
|\mathbb{E}_D \widetilde{k}_{x,h}| 
& \leq |\mathbb{E}_D \widetilde{k}_{y_j,h}| + 2 \varepsilon
\\
& \leq \sqrt{a  \int_{\mathbb{R}^d} k_{y_j,h}^2(x') \, \mathrm{d}P(x')} + \frac{ 2 a K(0)}{h^d} + \varepsilon
\\
& \leq \sqrt{a \int_{\mathbb{R}^d} k_{x,h}^2(x') \, \mathrm{d}P(x') }
       + \frac{\sqrt{a} \varepsilon}{2}
       + \frac{ 2 a K(0)}{h^d} + \varepsilon.
\end{align*}
Consequently we have
\begin{align*}
\int_{B_r} |\mathbb{E}_D \widetilde{k}_{x,h}| \, \mathrm{d}x
& \leq \int_{B_r} \sqrt{a \int_{\mathbb{R}^d} k_{x,h}^2(x') \, \mathrm{d}P(x') }  \, \mathrm{d}x
\\
& \phantom{=}  + r^d \lambda^d(B_1)\cdot \frac{ 2 a K(0) }{h^d}
               + r^d \lambda^d(B_1) ( \sqrt{a}/2 + 1 ) \varepsilon.
\end{align*}
Now recall that for $E\subset \mathbb{R}^d$ and $g:E \rightarrow \mathbb{R}$, H\"{o}lder's inequality implies
\begin{align*}
 \|g\|_{\frac{1}{2}}
  = \left( \int_{\mathbb{R}^d} |\eins_{E}|^{\frac{1}{2}} |g|^{\frac{1}{2}} \, \mathrm{d}x \right)^2
 \leq \int_{\mathbb{R}^d} |\eins_{E}|\, \mathrm{d}x
        \int_{\mathbb{R}^d} |g| \, \mathrm{d}x
 = \mu(E) \cdot \|g\|_1.
\end{align*}
This tells us that
\begin{align*}
\int_{B_r} \sqrt{a \int_{\mathbb{R}^d} k_{x,h}^2(x') \, \mathrm{d}P(x') }  \, \mathrm{d}x
\leq \sqrt{\mu(B_r)} \cdot 
       \sqrt{\int_{B_r} a \int_{\mathbb{R}^d} k_{x,h}^2(x') \, \mathrm{d}P(x') \, \mathrm{d}x}.
\end{align*}
Moreover, there holds
\begin{align*}
\int_{B_r} \int_{\mathbb{R}^d} k_{x,h}^2(x') \, \mathrm{d}P(x') \, \mu(\mathrm{d}x) 
& = \int_{\mathbb{R}^d} \int_{B_r} 
    h^{-2d}  K^2 \left( \|x - x'\|/h	 \right)  \, \mathrm{d}x \, \mathrm{d}P(x')
\\
& \leq \int_{\mathbb{R}^d} \int_{\mathbb{R}^d} h^{-2d} K^2 \left( \|x\|/h \right)
\, \mathrm{d}x \, \mathrm{d}P(x')
\\
& = h^{-d} \int_{\mathbb{R}^d} K^2(\|x\|) \, \mathrm{d}x
\\
& \leq K(0) h^{-d}.
\end{align*}
We now set $\varepsilon = \frac{1}{n}$ and obtain 
$\log (4 m) \leq c \log \frac{n r}{h}$. Thus we have
\begin{align}\label{intermediate}
\begin{split}
\int_{B_r} |\mathbb{E}_D \widetilde{k}_{x,h}| \, \mathrm{d}x
& \lesssim    \sqrt{\frac{(\log n)^{2/\gamma} r^d (\tau + \log (4 m))}{h^d n}}
        + \frac{(\log n)^{2/\gamma}  r^d (\tau + \log (4 m))}{h^d n} 
\\
&   \phantom{=} +  \sqrt{\frac{(\log n)^{2/\gamma}  (\tau + \log (4 m))}{n}} \cdot \frac{r^d}{n}
\\
& \lesssim      \sqrt{\frac{(\log n)^{2/\gamma} r^d (\tau + \log \frac{n r}{h})}{h^d n}}
        + \frac{(\log n)^{2/\gamma} r^d (\tau + \log \frac{nr}{h})}{h^d n} .
\end{split}
\end{align}
Now we need to estimate the corresponding integral over 
$H_r$. By definition we have
\begin{align*}
\int_{H_r} |\mathbb{E}_D \widetilde{k}_{x,h}| \, \mathrm{d}x 
\leq \int_{H_r} \mathbb{E}_D k_{x,h} \, \mathrm{d}x 
+ \int_{H_r} \mathbb{E}_P k_{x,h} \, \mathrm{d}x.
\end{align*}
From Lemma \ref{GeneralRandEstimate} we obtain
\begin{align*}
\int_{H_r} \mathbb{E}_D k_{x,h} \, \mathrm{d}x
\lesssim  D (H_{r/2})
       +   \biggl( \frac{h}{r} \biggr)^{\beta}
= \frac{1}{n} \sum_{i=1}^n \eins_{H_{r/2}}(x_i)
      +   \biggl( \frac{h}{r} \biggr)^{\beta}, 
\end{align*}
and
\begin{align*}
\int_{H_r} \mathbb{E}_P k_{x,h} \, \mathrm{d}x
\lesssim  P (H_{r/2})
     +   \biggl( \frac{h}{r} \biggr)^{\beta}.
\end{align*}
Since $r \geq 1$, we can construct a function $g$ with $\eins_{H_{r/2}} \leq g \leq \eins_{H_{r/4}}$
and there exists a function $\psi(r)$ such that $\vertiii{g} \leq \psi(r)$.
Applying Bernstein inequality in Theorem \ref{bernsteininequality} with respect to this function $g$, 
it is easy to see that when $n\geq n_2$, with probability $\mu$ at least $1 - 2e^{-\tau}$, there holds 
\begin{align*} 
\mathbb{E}_D g - \mathbb{E}_P g
\leq  \sqrt{\frac{8\tau(\log n)^{2/\gamma}}{n}}+\frac{8\tau(\log n)^{2/\gamma}}{3n},
\end{align*} 
where 
\begin{align*} 
n_2 := \max \left\{ \min \biggl\{ m \geq 3 : m^2 \geq 808 c_0 (3 \psi(r) + 1) \text{ and } \, \frac{m}{(\log m)^{\frac{2}{\gamma}}} \geq 4 \biggr\}, e^{\frac{3}{b}} \right\}.	
\end{align*}
This implies that with probability $\mu$ at least $1 - 2 e^{-	\tau}$, there holds
\begin{align*}
 D (H_{r/2})
 = \frac{1}{n} \sum_{i=1}^n \eins_{H_{r/2}}(x_i) 
 & \leq \mathbb{E}_D g
 \\
 & \leq \mathbb{E}_P g + \sqrt{\frac{8\tau(\log n)^{2/\gamma}}{n}}+\frac{8\tau(\log n)^{2/\gamma}}{3n}
 \\
 & \leq \mathbb{E}_P \eins_{H_{r/4}}(x_i) + \sqrt{\frac{8\tau(\log n)^{2/\gamma}}{n}}+\frac{8\tau(\log n)^{2/\gamma}}{3n}
\end{align*}
and consequently we obtain
\begin{align}\label{intermediate_II}
\int_{H_r} |\mathbb{E}_D \widetilde{k}	_{x,h}| \, \mathrm{d}x \lesssim   P (H_{r/4})
     + \sqrt{\frac{32\tau(\log n)^{2/\gamma}}{n}}+    \biggl( \frac{h}{r} \biggr)^{\beta}.	
\end{align}
By combining estimates in \eqref{intermediate} and \eqref{intermediate_II}, and taking $n_0=\max\{n_1,n_2 \}$, we have accomplished  the proof of Theorem \ref{OracleInequalityL1}.
\end{proof}

\begin{remark}\label{remark_on_cx}
Let us briefly discuss the choice of the function $\psi(r)$ in the proof of Theorem \ref{OracleInequalityL1}.
For example, in the case $\mathcal{C}(X) = \mathrm{Lip}(\mathbb{R})$, we can choose
\begin{align*} 
 g(x) :=
 \begin{cases}
  1, & \text{ for } |x| > r, \\
  0, & \text{ for } |x| < r/4, \\
  - \frac{4x}{3r} - \frac{1}{3}, & \text{ for } -r \leq x \leq -r/4, \\
  \frac{4x}{3r} - \frac{1}{3}, & \text{ for } r/4 \leq x \leq r.
 \end{cases}
\end{align*}
Then we have $\vertiii{g} \leq \frac{4}{3r} \leq 4/3$ and therefore, $n_2$ is well-defined. 
Moreover, it is easily seen that even for smoother underlying functions classes
like $C^1$ we can construct a function $g$ such that $\vertiii{g} < \infty$.
\end{remark}

\begin{proof}[of Theorem \ref{OracleInequalityInftyNorm}]
Recalling the definitions of $k_{x,h}$ and $\widetilde{k}_{x,h}$ given in \eqref{k_x_h} and \eqref{k_x_h_mean}, we have
\begin{align*}
\|f_{D,h} - f_{P,h}\|_{\infty}
= \sup_{x\in\Omega} |\mathbb{E}_D \widetilde{k}_{x,h}|.
\end{align*}
To prove the assertion, we first estimate $\mathbb{E}_D f_{x,h}$ for fixed  $x \in \mathbb{R}^d$ 
using the Bernstein inequality in Theorem \ref{bernsteininequality}. 
For this purpose, we first verify the following conditions: 
Obviously, we have $\mathbb{E}_P \widetilde{k}_{x,h} = 0$. Then, simple estimates imply
\begin{align*}
\|\widetilde{k}_{x,h}\|_{\infty}
\leq 2 \|k_{x,h}\|_{\infty}
\leq 2 h^{-d} \|K\|_{\infty}
\leq 2 h^{-d} K(0)
\end{align*}
and
\begin{align*}
\mathbb{E}_P \widetilde{k}_{x,h}^2
\leq \mathbb{E}_P k_{x,h}^2
= h^{-d} \int_{\mathbb{R}^d} K^2 \left( \|x - x'\|/h\right) f(x') h^{-d} \, \mathrm{d}x'
\lesssim \|f\|_{\infty} h^{-d}.
\end{align*}
Finally, the first condition in Assumption \ref{seminorm_assump} and Condition $(iii)$ in Assumption \ref{assump_kernel}
yield
\begin{align*}
 \vertiii{\widetilde{k}_{x,h}} 
       \leq \vertiii{k_{x,h}} 
       \leq h^{-d} \sup_{x \in \mathbb{R}^d} \vertiii{K \left(\|x - \cdot\|/h \right)}
       \leq h^{-d}\varphi(h).
\end{align*}
Therefore, we can apply the Bernstein inequality in Theorem \ref{bernsteininequality} and obtain that 
for $n \geq n_0^*$, for any fixed $x \in \mathbb{R}^d$, with probability $\mu$ at least  $1-4 e^{-\tau}$, there holds 
\begin{align} \label{FirstEstimateInftyNorm}
|\mathbb{E}_D \widetilde{k}_{x,h}|
\lesssim \sqrt{\frac{  \tau \|f\|_{\infty}(\log n)^{2/\gamma}}{h^d n}} + \frac{ K(0) \tau (\log n)^{2/\gamma}}{3 h^d n},
\end{align}
where
\begin{align}\label{nzerostar}
n_0^* := \max \left\{ \min \biggl\{ m \geq 3 : 
m \geq \biggl( \frac{808 c_0 (3  h^{-d}\varphi(h) +   K(0))}{2K(0)} \biggr)^{\frac{1}{d+1}}
\text{ and } 
\frac{m}{(\log m)^{\frac{2}{\gamma}}} \geq 4 \biggr\},
e^{\frac{d+1}{b}} \right\}.
\end{align}

Let us  consider the following function set
\begin{align*}
\mathcal{K}^\prime_{h,r_0} := \bigl\{ \widetilde{k}_{x,h} : x \in B_{r_0} \bigr\} 
\end{align*}
and choose $y_1, \ldots, y_m \in B_{r_0}$ such that $\{k_{y_1,h}, \ldots, k_{y_m,h}\}	$ is a minimal $\varepsilon/2$-net of $\mathcal{K}_{h,r_0}$ with respect to $\|\cdot\|_{\infty}$ and  $m = \mathcal{N} (\mathcal{K}_{h,r_0}, \|\cdot\|_\infty, \frac{\varepsilon}{2})$. As in the proof of Theorem \ref{OracleInequalityL1}, one can show that $\widetilde{k}_{y_1,h}, \ldots, \widetilde{k}_{y_m,h}$ is an $\varepsilon$-net of $\mathcal{K}^\prime_{h,r_0}$. Again from Proposition \ref{kernelSpaceCN} we know that there holds $\log (4 m) \lesssim \log \frac{r_0}{h \varepsilon}$. This in connection with  \eqref{FirstEstimateInftyNorm} implies that the following union bound  
\begin{align*}  
\sup_{j=1,\ldots,m} |\mathbb{E}_D \widetilde{k}_{y_j,h}|
\lesssim \sqrt{\frac{   \|f\|_{\infty} (\tau + \log ( 4 m))(\log n)^{2/\gamma}}{h^d n}}
  + \frac{  K(0) (\tau + \log (4 m))(\log n)^{2/\gamma}}{h^d n}
\end{align*}
holds with probability $\mu$ at least $1 - e^{-\tau}$. For any $x \in B_{r_0}$, there exists a $y_j$ such that
$\|k_{x,h} - k_{y_j,h}\|_{\infty} \leq \varepsilon$. Then we have
\begin{align*}
\bigl| |\mathbb{E}_D \widetilde{k}_{x,h}| - |\mathbb{E}_D \widetilde{k}_{y_j,h}| \bigr|
& \leq |\mathbb{E}_D \widetilde{k}_{x,h} - \mathbb{E}_D \widetilde{k}_{y_j,h}|
\\
& \leq |\mathbb{E}_D k_{x,h} - \mathbb{E}_D k_{y_j,h}|
       + |\mathbb{E}_P k_{x,h} - \mathbb{E}_P k_{y_j,h}|
\\
& \leq \|k_{x,h} - k_{y_j,h}\|_{L_1(D)} + \|k_{x,h} - k_{y_j,h}\|_{L_1(P)}
\\
& \leq  \varepsilon,
\end{align*}
and consequently with probability $\mu$ at least $1 - e^{-\tau}$, there holds 
\begin{align*}
|\mathbb{E}_D \widetilde{k}_{x,h}|
& \leq |\mathbb{E}_D \widetilde{k}_{y_j,h}| +  \varepsilon
\\
& \lesssim  \sqrt{\frac{  \|f\|_{\infty} (\tau + \log ( 4 m))(\log n)^{2/\gamma}}{h^d n}}
  + \frac{  K(0) (\tau + \log (4 m))(\log n)^{2/\gamma}}{h^d n} +    \varepsilon
\end{align*}
for any $x \in B_{r_0}$. By setting $\varepsilon = \frac{1}{n}$, we obtain
$\log (4 m) \lesssim  \log \frac{n r_0 }{h}	$. Thus, with probability $\mu$ at least $1 - e^{-\tau}$, we have
\begin{align*}
|\mathbb{E}_D \widetilde{k}_{x,h}|
& \lesssim  \sqrt{\frac{  \|f\|_{\infty} (\tau + \log ( \frac{n r_0 }{h}))(\log n)^{2/\gamma}}{h^d n}}
  + \frac{  K(0) (\tau + \log (\frac{n r_0 }{h}	))(\log n)^{2/\gamma}}{h^d n} +   \frac{1}{n}
\\
& \lesssim  \sqrt{\frac{  \|f\|_{\infty} (\tau + \log ( \frac{n r_0 }{h}))(\log n)^{2/\gamma}}{h^d n}}
  + \frac{  K(0) (\tau + \log (\frac{n r_0 }{h}	))(\log n)^{2/\gamma}}{h^d n}.
\end{align*}
By taking the supremum of the left hand side of the above inequality over $x$, we complete the proof of Theorem \ref{OracleInequalityInftyNorm}.
\end{proof}

\begin{proof}[of Theorem \ref{ConsistencyL1}]
Without loss of generality, we assume that $h_n \leq 1$. Since $h_n \to 0$, 
Theorem \ref{ApproximationError} implies that $\|f_{P,h} - f\|_1 \leq \varepsilon$.
We set 
\begin{align}\label{choice_radius}
r_n := \left(\frac{n h_n^d}{(\log n)^{(2+2\gamma)/\gamma}} \right)^{1/d} \to \infty
\end{align}
and we can also assume w.l.o.g that $r_n \geq 2$. Moreover, 
there exists a constant $n_1^\prime$ such that 
\begin{align*}
P (H_{r_n/2}) \leq \varepsilon,\,\,\forall\, n \geq n_1^\prime. 
\end{align*}

For any $0<\delta<1$, we select $\tau := \log(1/\delta)$.  Then there exists a constant $n_2^\prime$ such that 
$\log \frac{nr_n}{h_n} \geq \tau$
for all $n \geq n_2^\prime$.
On the other hand, with the above choice of $r_n$, we have
\begin{align*}
\log \frac{nr_n}{h_n}\leq \log\biggl(\frac{n^{1/d}h_n}{(\log n)^{(2+2\gamma)/\gamma}}\cdot \frac{n}{h_n}\biggr)\leq (1+d^{-1})\log n 
  \lesssim \log n.
\end{align*} 
Thus, for all $n \geq \max\{n_1^\prime,n_2^\prime\}$, we have
\begin{align*}
\frac{ (\log n)^{2/\gamma}r_n^d \log (\frac{nr_n}{h_n})}{n h_n^d}
\lesssim  \frac{ (\log n)^{2/\gamma}r_n^d \log n}{n h_n^d}=\frac{1}{\log n}
\to 0.
\end{align*}
Thus, following from Theorem \ref{OracleInequalityL1}, when $n$ is sufficient large, for any $\varepsilon>0$, with probability $\mu$ at least $1 - 3\delta$, there holds
\begin{align*}
\|f_{D,h_n} - f\|_1
\lesssim \varepsilon.
\end{align*}
Therefore, with properly chosen $\delta$, one can show that $f_{D,h_n}$ converges to $f$ under $L_1$-norm almost surely. We have completed the proof of Theorem \ref{ConsistencyL1}.
\end{proof}

\begin{proof}[of Theorem \ref{ConvergenceRatesL1}]
\textit{(i)} Combining the estimates in Theorem \ref{OracleInequalityL1} and Theorem \ref{ComparisionNorms}, we know that 	
with probability $\mu$ at least $1 - 2 e^{-\tau}$, there holds
\begin{align*}
\|f_{D,h} - f\|_1
& \lesssim \sqrt{\frac{(\log n)^{2/\gamma}r^d (\tau + \log(\frac{nr}{h_n}  ))}{h_n^d n}}
     + \frac{(\log n)^{2/\gamma}r^d (\tau + \log(\frac{nr}{h_n} ))}{h_n^d n}
\\
& \phantom{=} 
     +   \frac{\tau(\log n)^{2/\gamma}}{n} +   P \bigl( H_r \bigr)
     +   r^d h_n^{\alpha} + \biggl(\frac{h_n}{r}\biggr)^\beta
\\
& \lesssim \sqrt{\frac{(\log n)^{2/\gamma}r^d (\tau + \log(\frac{nr}{h_n}  	))}{h_n^d n}}
     + \frac{\tau(\log n)^{2/\gamma}}{n} +   P \bigl( H_r \bigr)
          +   r^d h_n^{\alpha} + \biggl(\frac{h_n}{r}\biggr)^\beta.
\end{align*}
Let $\tau := \log n$ and later we will see from the choices of $h_n$ and $r_n$ that there exists some constant $c$ such that
$\log(\frac{nr}{h} 	)$ can be bounded by $c  \log n$. Therefore, with probability $\mu$ at least $1-\frac{1}{n}$ there holds
\begin{align*}
\|f_{D,h} - f\|_1
& \lesssim   \sqrt{\frac{r^d (\log n)^{(2+\gamma)/\gamma}}{h_n^d n}}
     +  r^{- \eta d}
     +  r^d h_n^{\alpha}
\\
& \lesssim r^d \biggl( \frac{\log n}{n r^d} \biggr)^{\frac{\alpha}{2\alpha+d}} + r^{-\eta d}
\\
& \lesssim \left( \frac{(\log n)^{(2+\gamma)/\gamma}}{n} \right)^{\frac{\alpha\eta}{(1+\eta)(2\alpha+d)-\alpha}},
\end{align*}
by choosing
\begin{align*}
h_n = \biggl( \frac{(\log n)^{(2+\gamma)/\gamma}}{n} \biggr)^{\frac{1+\eta}{(1+\eta)(2\alpha+d)-\alpha}}
 \,\, \textrm{  and } \,\, 
 r:=r_n = \biggl( \frac{n}{(\log n)^{(2+\gamma)/\gamma}} \biggr)^{\frac{\alpha}{d(1+\eta)(2\alpha+d)-\alpha d}}.
\end{align*}

\textit{(ii)} 
Similar to case \textit{(i)}, one can show that with probability $\mu$ at least $1-\frac{1}{n}$ there holds
\begin{align*}
\|f_{D,h} - f\|_1
& \lesssim  \sqrt{\frac{r^d (\log n)^{(2+\gamma)/\gamma}}{h_n^d n}}
     +  e^{- a r^\eta}
     + r^d h_n^{\alpha}
\\
& \lesssim r^d \biggl( \frac{(\log n)^{(2+\gamma)/\gamma}}{n r^d} \biggr)^{\frac{\alpha}{2\alpha+d}} +  e^{- a r^\eta}
\\
& \lesssim	 \biggl( \frac{(\log n)^{(2+\gamma)/\gamma}}{n} \biggr)^{\frac{\alpha}{2\alpha+d}}
     (\log n)^{\frac{d}{\eta} \cdot \frac{\alpha+d}{2\alpha+d}},
\end{align*}
by choosing
\begin{align*}
 h_n = \biggl( \frac{(\log n)^{(2+\gamma)/\gamma}}{n} \biggr)^{\frac{1}{2\alpha+d}}
     (\log n)^{- \frac{d}{\eta} \cdot \frac{1}{2\alpha+d}}
 \,\, \textrm{  and } \,\, 
 r_n = (\log n)^{\frac{1}{\eta}}. 
\end{align*}

\textit{(iii)}
From Theorem \ref{ComparisionNorms} 
we see that with confidence $1-\frac{1}{n}$, there holds
\begin{align*}
\|f_{D,h} - f_{P,h}\|_1
& \lesssim  \sqrt{\frac{r_0^d (\log n)^{(2+\gamma)/\gamma}}{h_n^d n}}   
     +  h_n^\alpha
 \lesssim  \left( (\log n)^{(2+\gamma)/\gamma}/n \right)^{\frac{\alpha}{2\alpha+d}},
\end{align*}
where $h_n$ is chosen as  
\begin{align*}
 h_n = \left( (\log n)^{(2+\gamma)/\gamma}/n \right)^{\frac{1}{2\alpha+d}}. 
\end{align*}
The proof of Theorem \ref{ConvergenceRatesL1} is completed. 
\end{proof}

\begin{proof}[of Theorem \ref{ConvergenceRatesLInfty}]
The desired estimate is an easy consequence if we combine the estimates in Theorem \ref{OracleInequalityInftyNorm} and Theorem \ref{ApproximationError} \textit{(ii)} and choose 
\begin{align*}
h_n = \left(  (\log n)^{(2+\gamma)/\gamma}/n \right)^{\frac{1}{2\alpha+d}}.
\end{align*}
We omit the details of the proof here.
\end{proof}

\section{Conclusion}\label{sec::conclusion}
In the present paper, we studied the kernel density estimation problem for dynamical systems admitting a unique invariant Lebesgue density by using the $\mathcal{C}$-mixing coefficient to measure the dependence among observations. The main results presented in this paper are the consistency and convergence rates of the kernel density estimator in the sense of $L_1$-norm and $L_\infty$-norm. With properly chosen bandwidth, we showed that the kernel density estimator is universally consistent. Under mild assumptions on the kernel function and the density function, we established convergence rates for the estimator. For instance, when the density function is bounded and compactly supported, both $L_1$-norm and $L_\infty$-norm convergence rates with the same order can be achieved for general geometrically time-reversed $\mathcal{C}$-mixing dynamical systems. The convergence mentioned here is of type ``with high probability" due to the use of a Bernstein-type exponential inequality and this makes the present study different from the existing related studies. We also discussed the model selection problem of the kernel density estimation in the dynamical system context by carrying out numerical experiments.

\acks{The research leading to these results has received funding from the European Research Council under the European Union's Seventh Framework Programme (FP7/2007-2013) / ERC AdG A-DATADRIVE-B (290923). This paper reflects only the authors' views, the Union is not liable for any use that may be made of the contained information. Research Council KUL: GOA/10/09 MaNet, CoE PFV/10/002 (OPTEC), BIL12/11T; PhD/Postdoc grants. Flemish Government:  FWO: projects: G.0377.12 (Structured systems), G.088114N (Tensor based data similarity); PhD/Postdoc grants. IWT: projects: SBO POM (100031); PhD/Postdoc grants. iMinds Medical Information Technologies SBO 2014. Belgian Federal Science Policy Office: IUAP P7/19 (DYSCO, Dynamical systems, control and optimization, 2012-2017). The corresponding author is Yunlong Feng.}

\small{\bibliography{FENGBib}}
\end{document}